\documentclass{article}

\usepackage[final]{neurips_2021}

\usepackage[utf8]{inputenc} 
\usepackage[T1]{fontenc}    
\usepackage[colorlinks,citecolor=blue]{hyperref}       
\usepackage{url}            
\usepackage{booktabs}       
\usepackage{amsfonts}       
\usepackage{nicefrac}       
\usepackage{microtype}      
\usepackage{xcolor}         
\usepackage{enumitem}

\usepackage{amsmath}
\usepackage{amsthm}
\newtheorem{theorem}{Theorem}
\newtheorem{corollary}{Corollary}
\newtheorem{lemma}{Lemma}
\newtheorem{definition}{Definition}
\newtheorem{claim}{Claim}
\usepackage{multirow}
\usepackage{mathtools}
\usepackage[capitalize]{cleveref}
\usepackage{wrapfig}
\usepackage{algorithm2e}
\DeclareMathOperator{\E}{\mathbb{E}}
\DeclareMathOperator{\Nmax}{{N_{\max}}}
\DeclareMathOperator{\cT}{\mathcal{T}}
\DeclareMathOperator{\cO}{\mathcal{O}}

\DeclareMathOperator{\cX}{\mathcal{X}}
\DeclareMathOperator{\cA}{\mathcal{A}}

\DeclareMathOperator{\Ber}{\mathcal{B}}
\DeclareMathOperator{\Reg}{\operatorname{Reg}}
\DeclareMathOperator{\Regimp}{\Reg_{\operatorname{Imp}}}

\DeclareMathOperator*{\kl}{\operatorname{kl}}

\usepackage{color}
\usepackage[normalem]{ulem}

\newcommand{\commentout}[1]{}
\newcommand{\bbN}{\mathbb{N}}
\newcommand{\bbP}{\mathbb{P}}

\newcommand{\bbI}{\mathbb{I}}
\newcommand{\bbQ}{\mathbb{Q}}
\newcommand{\bbV}{\mathbb{V}}
\newcommand{\cE}{\mathcal{E}}
\newcommand{\cB}{\mathcal{B}}
\newcommand{\cN}{\mathcal{N}}
\newcommand{\C}{\mathfrak{C}}
\newcommand{\ip}[1]{\langle #1 \rangle}
\newcommand{\TV}[1]{\| #1 \|_{TV}}
\newcommand{\TVQP}[1]{\TV{\bbQ_{#1}-\bbP_{ #1} }}
\newcommand{\PolyLog}{\operatorname{PolyLog}}

\newif\ifsup

\title{The Pareto Frontier of model selection \\for general Contextual Bandits}

\author{%
  Teodor Marinov\thanks{Author was at Johns Hopkins University during part of this work.}\\
  Google Research\\
  \texttt{tvmarinov@google.com} \\
  \And
  Julian Zimmert\\
  Google Research\\
  \texttt{zimmert@google.com}\\
}

\begin{document}

\maketitle

\begin{abstract}
  Recent progress in model selection raises the question of the fundamental limits of these techniques. Under specific scrutiny has been model selection for general contextual bandits with nested policy classes, resulting in a COLT2020 open problem. It asks whether it is possible to obtain simultaneously the optimal single algorithm guarantees over all policies in a nested sequence of policy classes, or if otherwise this is possible for a trade-off  $\alpha\in[\frac{1}{2},1)$ between complexity term and time: $\ln(|\Pi_m|)^{1-\alpha}T^\alpha$. 
We give a disappointing answer to this question. Even in the purely stochastic regime, the desired results are unobtainable. We present a Pareto frontier of up to logarithmic factors matching upper and lower bounds, thereby proving that an increase in the complexity term $\ln(|\Pi_m|)$ independent of $T$ is unavoidable for general policy classes.
As a side result, we also resolve a COLT2016 open problem concerning second-order bounds in full-information games.
\end{abstract}

\section{Introduction}
\label{sec:intro}
Contextual multi-armed bandits are a fundamental problem in online learning~\citep{auer2002nonstochastic,langford2007epoch,chu2011contextual,abbasi2011improved}. The contextual bandit problem proceeds as a repeated game between a learner and an adversary. At every round of the game the adversary prepares a pair of a context and a loss over an action space, the learner observes the context and selects an action from the action space and then observes only the loss of the selected action. The goal of the learner is to minimize their cumulative loss. The performance measure, known as \emph{regret}, is the difference between the learner's cumulative loss and the smallest loss of a fixed policy, belonging to an apriori determined policy class, mapping contexts to actions.
Given a single contextual bandit instance with finite sized policy class, the well-known Exp4 algorithm \citep{auer2002nonstochastic} achieves the optimal regret
bound of $\cO(\sqrt{KT\ln(|\Pi|)}$.
Regret guarantees degrade with the complexity of the policy class, therefore a a learner might want to leverage ``guesses'' about the optimal policy.
Given policy classes $\Pi_1\subset \dots\subset \Pi_M$, a learner would ideally suffer regret scaling only with the complexity of $\Pi_{m^*}$, the smallest policy class containing the optimal policy $\pi^*$.
While these kind of results are obtainable in full-information games, in which the learner gets to observe the loss for all actions, they are impossible for multi-armed bandits~\citep{lattimore2015pareto}. 
In some aspects, contextual bandits are an intermediate setting between full-information and multi-armed bandits and it is unknown if model selection is possible.
\citet{FKL20} stated model selection in contextual bandits as a relevant open problem in COLT2020.
Any positive result for model selection in contextual bandits would imply a general way to treat multi-armed bandits with a switching baseline. 
Furthermore any negative result is conjectured to implicate negative results on another unresolved open problem on second order bounds for full-information games \citep{F16open}.

In this paper, we give a fairly complete answer to the questions above.
\begin{enumerate}[label=P.\arabic*]
    \item We provide a Pareto frontier of upper bounds for model selection in contextual bandits with finite sized policy classes.
    \item We present matching lower bounds that shows that our upper bounds are tight, thereby resolve the motivating open problems \citep{FKL20}.
    \item We present a novel impossibility result for adapting to the number of switch points under adaptive adversaries~\citep{besbes2014stochastic}.
    \item We negatively resolve an open problem on second order bounds for full-information \citep{F16open}. 
\end{enumerate}

\paragraph{Related work.}
A problem closely related to contextual bandits with finite policy classes are linear contextual bandits. 
Model selection in linear contextual bandit problems has recently received significant attention, however none of these resuls transfer to the finite policy case. In the linear bandits problem the $m$-th policy class is a subset of $\mathbb{R}^{d_m}$ and the losses $\ell_{t,\pi(x)},\pi \in \Pi_m$ are linear, that is $\ell_{t,\pi(x)} = \langle \theta_m, \phi_m(x,\pi(x))\rangle + \xi$. Here $\phi_m : \mathcal{X}\times\mathcal{A}\rightarrow \mathbb{R}^{d_m}$ is a feature embedding mapping from context-action pairs into $\mathbb{R}^{d_m}$, $\xi$ is mean-zero sub-Gaussian noise with variance proxy equal to one and $\theta_m \in \mathbb{R}^{d_m}$ is an unknown parameter.

\cite{foster2019model} assume the contexts are also drawn from an unknown distribution $x\sim\mathcal{D}$ and propose an algorithm which does not incur more than $\smash{\tilde O(\frac{1}{\gamma^3}(i^*T)^{2/3}(Md_{i^*})^{1/3})}$, where $\gamma^3$ is the smallest eigenvalue of the covariance matrix of feature embeddings $\Sigma = \mathbb{E}_{x\sim\mathcal{D}}\left[\frac{1}{M}\sum_{a\in\mathcal{A}}\phi_M(x,a)\phi_M(x,a)^\top\right]$. \cite{pacchiano2020model} propose a different approach based on the corralling algorithm of \cite{agarwal2016corralling} which enjoys a $\tilde O(d_{i^*}\sqrt{T})$ regret bound for finite action sets and  $\tilde O(d_{i^*}^2\sqrt{T})$ bound for arbitrary action sets $\mathcal{A}$. Later, \cite{pacchiano2020regret} design an algorithm which enjoys a gap-dependent guarantee under the assumption that all of the miss-specified models have regret $R_i(t) \geq \Delta t, \forall t\in [T]$. Under such an assumption, the authors recover a regret bounds of the order $\tilde O(d_{i^*}\sqrt{T} + d_{i^*}^4/\Delta)$ for arbitrary action sets. \cite{cutkosky2020upper} also manage to recover the $O(d_{i^*}\sqrt{T})$ and $O(d_{i^*}^2\sqrt{T})$ bounds for the model selection problems through their corralling algorithm. \cite{ghosh2021problem} propose an algorithm which enjoys $\tilde O\left(\frac{d_M^2}{\gamma^{4.65}} + \sqrt{d_{m^*}T}\right)$ in the finite arm setting, where $\gamma = \min\{|\theta_{m^*,i}| : |\theta_{m^*,i}| > 0\}$ is the smallest, in absolute value, entry of $\theta_{m^*}$. Their algorithm also enjoys a similar guarantee for arbitrary action sets with $\sqrt{d_{m^*}T}$ replaced by $d_{m^*}\sqrt{T}$. \cite{zhu2021pareto} show that it is impossible to achieve the desired regret guarantees of $\sqrt{d_{m^*}T}$ without additional assumptions by showing a result similar to the one of \cite{lattimore2015pareto}. The work of \cite{lattimore2015pareto} states that in the stochastic multi-armed bandit problem it is impossible to achieve $\sqrt{T}$ regret to a fixed arm, without suffering at least $K\sqrt{T}$ regret to a different arm. 

\cite{chatterji2020osom} study the problem of selecting between an algorithm for the linear contextual bandit problem and the simple stochastic multi-armed bandit problem, that is they aim to achieve simultaneously a regret guarantee which is instance-dependent optimal for the stochastic multi-armed bandit problem and optimal for the finite arm stochastic linear bandit problem. The proposed results only hold under additional assumptions. More generally, the study of the corralling problem, in which we are presented with multiple bandit algorithms and would like to perform as well as the best one, was initiated by \cite{agarwal2016corralling}. Other works which fall into the corralling framework are that of \cite{FGMZ20} who study the miss-specified linear contextual bandit problem, that is the observed losses are linear up to some unknown $\epsilon$ miss-specification, and the work of \cite{arora2021corralling} who study the corralling problem for multi-armed stochastic bandit algorithms.

Our work also shows an impossibility result for the stochastic bandit problem with non-stationary rewards. \cite{auer2002using} first investigates the problem under the assumption that there are $L$ distributional changes throughout the game and gives an algorithm with a $\tilde O(\sqrt{KLT})$ \emph{dynamic regret}\footnote{In dynamic regret the comparator is the best action for the current distribution.} bound, under the assumption that $L$ is known. \cite{auer2019adaptively} achieves similar regret guarantees without assuming that the number if switches (or changes) of the distribution is known. A different measurement of switches is the total variation of changes in distribution $V_T = \sum_{t=2}^T \|\mathbb{E}[\ell_t] - \mathbb{E}[\ell_{t-1}] \|_{\infty}$. Multiple works give dynamic regret bounds of the order $\tilde O(V_T^{1/3}T^{2/3})$ (hiding dependence on the size of the policy class) when $V_T$ is known, including for extensions of the multi-armed bandit problem like contextual bandits and linear contextual bandits~\citep{besbes2014stochastic,luo2018efficient,besbes2015non,wei2017tracking}. \cite{cheung2019learning,zhao2020simple} further show algorithms which enjoy a parameter free regret bound of the order $\tilde O(V_T^{1/4}T^{3/4})$ (hiding dependence on dimensionality) for the linear bandits problem.
The lower bound in Table~\ref{table:1} might seem to contradict such results. In Section~\ref{sec:s-switch_comment} we carefully explain why this is not the case.

Finally, our lower bounds apply to the problem of devising an algorithm which simultaneously enjoys a second order bound over any fraction of experts. \cite{cesa2007improved} first investigate the problem of second order bounds for the experts problem, in which the proposed algorithm maintains a distribution $p_t$ over the set of $K$ experts, during every round of the game. The experts are assumed to have stochastic losses $\ell_t$ and the work shows an algorithm with $\tilde O(\sqrt{\sum_{t=1}^T \mathbb{V}_{i\sim p_t}[\ell_{t,i}]\log{K}})$ regret guarantee. \cite{chaudhuri2009parameter,chernov2010prediction,luo2015achieving,koolen2015second} study a different experts problem in which the comparator class for the regret changes from the best expert in hindsight to the uniform distribution over the best $\lfloor\epsilon K\rfloor$ experts for an arbitrary positive $\epsilon$. The above works propose algorithms which achieve a $\tilde O(\sqrt{T\log(1/\epsilon)})$ regret bound for all $\epsilon$ simultaneously. \cite{F16open} asks if there exists an algorithm which enjoys both guarantees at the same time, that is, does there exist an algorithm with regret bound $\tilde O(\sqrt{\sum_{t=1}^T \mathbb{V}_{i\sim p_t}[\ell_{t,i}]\log(1/\epsilon)})$ which holds simultaneously for all positive $\epsilon$.

\begin{table}
\centering
\begin{tabular}{p{1.72cm}|c|c }
  General CB & Upper bound & Lower bound  \\ [0.5ex] 
 \hline
 {\bf adaptive adversary} & $\cO(\max\{\C,\frac{\ln|\Pi_m|}{\C}\}\sqrt{MKT})$ & $\Omega(\max\{\C,\frac{\ln|\Pi_m|}{\C}\}\sqrt{\frac{KT}{\ln(K)}})$  \\ 
 \hline
 {\bf oblivious adversary / stochastic } & $\cO(\max\{\C,\frac{\ln|\Pi_m|}{\C}\}\sqrt{MKT})$ &   $\Omega(\max\{\C,\frac{\ln|\Pi_m|}{\C}\}\sqrt{T})$  \\[4em]
  S-switch & Upper bound & Lower bound  \\ [0.5ex] 
 \hline
 {\bf adaptive adversary} & $\tilde\cO(\max\{\C,\frac{S}{\C}\}\sqrt{KT})$ & $\Omega(\max\{\C,\frac{S}{\C}\}\sqrt{KT})$  \\ 
 \hline
 oblivious adversary & $\tilde\cO(\sqrt{SKT}+T^{3/4})$ &   $\Omega(\sqrt{SKT})$  \\
 \hline
 stochastic & $\tilde\cO(\sqrt{SKT})$ & $\Omega(\sqrt{SKT})$
\end{tabular}
\caption{Overview of our results. Our novel contributions are in bold; lower bounds only hold if the expressions are not exceeding $\Theta(T)$. The stochastic/oblivious lower bounds hold only for proper algorithms.}
\label{table:1}
\end{table}
\paragraph{Notation.}
For any $N\in\bbN$, $[N]$ denotes the set $\{1,\dots,N\}$. $\tilde O$ notation hides poly-logarithmic factors in the horizon $T$ and the number of arms $K$ but not in the size of the policy classes $|\Pi_m|$.

\section{Problem setting}
We consider the contextual bandit problem with general policy classes of finite size. There are $K$ arms and nested policy classes $(\Pi_m)_{m=1}^M$, where a policy $\pi\in\Pi_m, \pi:\cX\rightarrow[K]$ is a mapping from an arbitrary context space $\cX$ into the set of $K$ arms.
The game is played for $T$ rounds and at any time $t$, the agent observes a context $x_t\in\cX$, selects arm $A_t\in[K]$ and observes the loss $\ell_{t,A_t}$ from an otherwise unobserved loss vector $\ell_t\in[K]$.
We measure an algorithm's performance in terms of pseudo-regret, which is the expected cumulative regret of the player against following a fixed policy in hindsight
\begin{align*}
    \Reg(T,\Pi) = \max_{\pi\in\Pi}\E\left[\sum_{t=1}^T\ell_{t,A_t}-\ell_{t,\pi(x_t)}\right]\,.
\end{align*}

\paragraph{Environments.}
We distinguish between \emph{stochastic} environments and \emph{oblivious} or \emph{adaptive} adversaries.
In stochastic environments, there are unknown distribution $P_{\cX}, Q$ such that $x_t\sim P_{\cX}$ and $\ell_t\sim Q(\cdot | x_t)$ are i.i.d.\ samples.
In the adversarial regime, the distributions can change over time, i.e. $x_t\sim P_{\cX,t}, \ell_t\sim Q_t(\cdot|x_t)$.
When the choices are fixed at the beginning of the game, the adversary is called oblivious, while an adaptive adversary can chose $P_{\cX,t}, Q_t$ based on all observations up to time $t-1$.

Often the stochastic-adversarial hybrid problem has been studied with adversarially chosen context but stochastic losses. 
In our work, all upper bounds hold in the stronger notion where both the losses and the contexts are adaptive, while the lower bounds hold for the weaker notion where only the contexts are adaptive.

\paragraph{Open problem \citep{FKL20}.}
The regret upper bounds for all regimes introduced above for a fixed policy class $\Pi$ of finite size are of the order $\tilde O(\sqrt{\ln(|\Pi|)KT})$ and can be achieved by the Exp4 algorithm \citep{auer2002nonstochastic}.
The question asked by \citet{FKL20}: 
For a nested sequence of policies $\Pi_1\subset\Pi_2\subset\dots\subset\Pi_M$, is there a universal $\alpha\in[\frac{1}{2},1)$ such that a regret bound of
\begin{align}
\label{eq:open}
    \Reg(T,\Pi_m) = \PolyLog(K,M)\tilde\cO\left(\ln(|\Pi_m|)^{1-\alpha}T^\alpha \right)
\end{align}
is obtainable for all $m\in[M]$ simultaneously?

W.l.o.g.\ we can assume that $M=\cO(\ln\ln(|\Pi_M|))=\cO(\ln(T))$.
Otherwise we take a  subset of policy classes that includes $\Pi_M$ and where two consequent policy classes at least square in size.
Due to nestedness, any guarantees on this subset of models imply up to constants the same bounds on the full set. 

\paragraph{S-switch}
A motivating example for studying nested policy classes is the S-switch problem.
The context is simply $x_t=t\in[T]$ and the set of policies is given by 
\begin{align*}
    \Pi_S = \left\{\pi \bigg\vert \sum_{t=1}^{T-1}\bbI\{\pi_t\neq \pi_{t+1}\}\leq S\right\}\,,
\end{align*}
the set of policies that changes its action not more than $S$ many times.
Any positive result for contextual bandits with finite sized policy classes would provide algorithms that adapt to the number of switch points, since $\ln|\Pi_S| = \tilde \cO(S)$. To make clear what problem we are considering, we are using $\Reg_{SW}(T,S)$ to denote the regret in the switching problem.

Next, we define the class of \emph{proper} algorithms which choose their policy at every time step $t$ independently of context $x_t$. Restricting our attention to such algorithms greatly reduces the technicalities for lower bound proofs in the non-adaptive regimes. The lower bound for this class of algorithms is also at the core of the argument for adaptive (improper) algorithms in stochastic environments.

\begin{definition}
We call an algorithm \emph{proper}, if at any time $t$, the algorithm follows the recommendation of a policy $\pi_{i_t} \in \Pi_M$, and if
the choice of $i_t$ by the algorithm, is independent of the context $x_t$.
\end{definition}
\paragraph{Example.} EXP4 is proper.

The properness assumption intuitively allows us to reduce the model selection problem to a bandit-like problem in the space of all policies $\Pi_M$. We give more details in Section~\ref{sec:stocastic lower bound} and Appendix~\ref{app:lower_bounts_proper_alg}.

\section{Upper bounds}
\label{sec:upper}
In this section, we generalize the Hedged-FTRL algorithm \citep{FGMZ20} to obtain an upper bound for model selection over a large collection of $\sqrt{T}$ regret algorithms.

\begin{theorem}
\label{thm:upper main}
    For any $\C>0$, we can tune Hedged-FTRL over a selection of $M$ instances of EXP4 operating on policy classes $\Pi_1,\dots\Pi_M$, such that the following regret bound holds uniformly over all $m\in[M]$
    \begin{align*}
        \Reg(T,\Pi_m) = \tilde\cO\left(\max\left\{\C,\frac{\ln|\Pi_m|}{\C}\right\}\sqrt{MKT}\right)\,.
    \end{align*}
\end{theorem}
\begin{wrapfigure}[13]{r}{6.5cm}
\vspace{-0.15cm}
\begin{algorithm}[H]
\caption{Hedged FTRL}
\label{alg:dylan}
\DontPrintSemicolon
\LinesNumberedHidden
\KwIn{$\alpha, R, \beta, Top, (Base_i)_{i=1}^M$}
\For{$t= 1,\dots,T$}{
Get $M_t, q_{t,M_t}$ from $Top$\;
Let $Base_{M_t}$ play the next round and receive $A_t$\;
Play $A_t$ and observe $\ell_{t,A_t}$\;
Update $Base_{M_t}$ with $\ell_{t,A_t}/q_{t,M_t}$\;
Update $Top$ with $(M_t,\ell_t)$\;
\If{$q_{t+1}$ would violate \cref{eq:hedge constraint}}{
Bias losses by $b_t$ to ensure \cref{eq:hedge constraint}.
}
}
\end{algorithm}
\end{wrapfigure}
\paragraph{Hedged-FTRL.}
$(\alpha,R)$-hedged FTRL, introduced in \citet{FGMZ20}, is a type of Follow the Regularized Leader (FTRL) algorithm which is used as a corralling algorithm~\citep{agarwal2016corralling}. At every round $t$, the algorithm chooses to play one of $M$ base algorithms $(Base_i)_{i=1}^M$. Base algorithm $i$ is selected with probability $q_{t,i}$, where $q_t \in \Delta^{M-1}$ is a distribution over base algorithms determined by the FTRL rule $q_t = \arg\min_{q \in \Delta^{M-1}} \langle q, L_t - B_t \rangle + F(q)/\eta$, where $L_t \in \mathbb{R}^M_{+}$ is the sum of the loss vectors $(\mathbf{e}_{M_s}\ell_{s,A_s}/q_{s,M_s})_{s=1}^{t-1}$, $F:\Delta^{M-1}\to \mathbb{R}$ is the potential induced by the $\alpha$-Tsallis entropy, $\eta$ is a step size determined by the problem parameters and $B_t$ is a special bias term which we now explain. Define $\rho_{t,m}^{-1} = \min\{\beta_m,\min_{s\in[t]}q_{s,m}\}$, and initialize $B_{0,m}=\rho_{1,m}^\alpha R_m$. Here $\rho_t$ is a vector which tracks the variance of the loss estimators, $R$ is a vector with regret upper bounds for the base algorithms, and $\beta\leq q_1$ is a threshold depending on the base algorithms.
At any time $t$, after selecting base algorithm $M_t$ to play the current action, the top (corralling) algorithm observes its loss and gives as feedback an important weighted loss to the selected base, $M_t$.
Whenever the base played at round $t$ would satisfy $\rho_{t+1,M_t}>\rho_{t,M_t}$, the loss fed to the top algorithm is adjusted with a bias $b_{t,M_t}$, such that the cumulative biases track the quantity $\rho_{t+1,m}^{\alpha} R_m$. This has been shown to be always possible~\citep{FGMZ20}.
The condition for adjusting the biases reads
\begin{align}
\label{eq:hedge constraint}
    \forall m\in[M]:\,B_{0,m}+\sum_{s=1}^tb_{s,m}=\rho_{t+1,m}^{\alpha} R_{m}\,.
\end{align}

The condition in Equation~\ref{eq:hedge constraint} is motivated in a similar way to the stability condition in the work of \citet{agarwal2016corralling}. Algorithm~\ref{alg:dylan} constructs an unbiased estimator for the loss vector, $\mathbf{e}_{M_t}\ell_{t,A_t}/q_{t,M_t}$, and updates each base algorithm accordingly. A similar update is present in the \textsc{Corral} algorithm~\cite{agarwal2016corralling} in which each of the base learners also receives an importance weighted loss. The regret of the base learners is assumed to scale with the variance of the importance weighted losses. This assumption is natural and in practice holds for all bandit or expert algorithms. The scaling of the regret, however, must be appropriately bounded as \citet{agarwal2016corralling} show, otherwise no corralling or model selection guarantees are possible. Formally, the following stability property is required. If an algorithm $\mathcal{B}$ enjoys a regret bound $R$ under environment $\mathcal{V}$ with loss sequence $(\ell_t)_{t=1}^T$, then the algorithm is \emph{$(\alpha,R)$-stable} if it enjoys a regret bound of the order $\mathbb{E}[\rho_{\max}^{\alpha}]R$ under the environment $\mathcal{V}'$ of importance weighted losses $(\hat\ell_t)_{t=1}^T$, where $\rho_{\max}$ is the maximum variance of the $T$ losses and the expectation is taken with respect to any randomness in $\mathcal{B}$. Essentially all bandit and expert algorithms used in practice are $(\alpha,R)$-stable with $\alpha \leq 1/2$, e.g., Exp4 is $(1/2,\sqrt{KT\ln(|\Pi|)})$-stable. The bias terms in Algorithm~\ref{alg:dylan} intuitively cancel the additional variance introduced by the importance weighted losses and this is why we require the biases to satisfy Equation~\ref{eq:hedge constraint}.

\begin{theorem}
\label{thm:upper general}
    Given a collection of base algorithms $(\cB_m)_{m=1}^M$ which are $(1/2,\sqrt{\C_m T})$-stable, that is
    \begin{align*}
        \forall m\in[M]:\,\Regimp(T,\cB_m) \leq \E[\sqrt{\rho_{Tm}}]\sqrt{\C_mT}\,,
    \end{align*}
    and any $\C\geq0$,
    then the regret of ($1/2,R,\beta)$-hedged Tsallis-Inf with $R_m = \sqrt{\C_mT}$, $\beta_m = \frac{1}{M}\max\{1,\frac{\C^2}{\C_m}\}$ satisfies a simultaneous regret of
    \begin{align*}
        \forall m\in[M]:\,\Reg(T,\cB_m) \leq 2\max\left\{\C,\frac{\C_m}{\C}\right\}\sqrt{MT} + \sqrt{2MT}\,.
    \end{align*}
\end{theorem}
The analysis follows closely the proof of \citet{FGMZ20} and is postponed to \ifsup the supplementary material\else \cref{app:upper}\fi.

\cref{thm:upper general} recovers the bounds of \citet{pacchiano2020model} for model selection in linear bandits, but holds in more general settings including adaptive adversaries in both contexts and losses. It neither requires nestedness of the policies nor that the policies operate on the same action or context space.
\begin{proof}[Proof of \cref{thm:upper main}]
 The EXP4 algorithm initialized with policy class $\Pi_m$ satisfies the condition of \cref{thm:upper general} with $\C_m=\cO(\ln|\Pi_m|)$, as shown in \citet{agarwal2016corralling}.
 Hence \cref{thm:upper main} is a direct corollary of \cref{thm:upper general}.
\end{proof}

\section{Lower bounds}
\label{sec:lower}
We present lower bounds that match the upper bounds from \cref{sec:upper} up to logarithmic factors, thereby proving a tight Pareto frontier of worst-case regret guarantees in model selection for contextual bandits.

In the first part of this section, we consider a special instance of $S$-switch with \emph{adaptive} adversary.
The proof technique based on Pinsker's inequality is folklore and leads to the following theorem.

\begin{theorem}
\label{thm:adaptive simple}
For any $K\geq 3$, sufficiently large $T$, and any algorithm with regret guarantee
\begin{align*}
    \Reg_{SW}(T,1) = \cO(\C\sqrt{KT})\,, 
\end{align*}
there exists for any number of switches $S=\Omega(\C^2)$ a stochastic bandit problem such that 
\begin{align*}
    \Reg_{SW}(T,S) = \Omega\left(\min\left\{\frac{S}{\C}\sqrt{KT},T\right\}\right) \,.
\end{align*}
This bound holds even when the agent is informed about the number of switches up to time $t$.
\end{theorem}
Since this bound holds even when the agent is informed about when a switch occurs, we can restrict the policy class to policies that only switch arms whenever the agent is informed about a switch in the environment. This as a contextual bandit problem with context $\cX=[S+1]$ and $|\Pi_S|=\Theta(K^S)$ policies. Hence \cref{thm:adaptive simple} implies a lower bound of 
    $\Reg(T,\Pi_S) = \Omega\left(\min\left\{\frac{\ln|\Pi_S|}{\C\ln(K)}\sqrt{KT},T\right\}\right)\,.$
In the second part of the section, we consider the \emph{stochastic} regime.
Our lower bound construction is non-standard and relies on bounding the total variation between problem instances directly without the use of Pinsker's inequality.
\begin{theorem}
\label{thm:stochastic simple}
There exist policy classes $\Pi_1\subset\Pi_2$
\footnote{
In \citet{FKL20} open problem 2, they ask about model based contextual bandit with realizability. Our lower bound is providing an instance of that.
}
with $|\Pi_2| =\Omega(\C^2)$,  such that if the regret of a proper algorithm is upper bounded in any environment by
\begin{align*}
    \Reg(T,\Pi_1) = \cO(\C\sqrt{T})\,,
\end{align*}
then there exists an environment such that
\begin{align*}
    \Reg(T,\Pi_2) = \Omega\left(\max\left\{\C,\frac{\ln|\Pi_2|}{\C}\right\}\sqrt{T}\right)\,.
\end{align*}
\end{theorem}

These theorems directly provide negative answers to \citep{FKL20}.

\begin{corollary}
\label{cor:main}
There is no $\alpha\in[\frac{1}{2},1)$ that satisfies the regret guarantee of open problem \eqref{eq:open} for any algorithm in the adaptive adversarial regime or any proper algorithm in the stochastic case.
\end{corollary}
\begin{proof}
By \cref{thm:adaptive simple,thm:stochastic simple}, for any $\alpha>0$ there exists $K=3,\,M=2,\,|\Pi_1|=1, |\Pi_2|=\Theta(\exp(T^{\alpha}))$. Assume that $\Reg(T,\Pi_1) \leq C_TT^\alpha=C_TT^{\alpha-\frac{1}{2}}\sqrt{T}\,,$
where $C_T=\PolyLog(T)$.
Hence by \cref{thm:adaptive simple} and \cref{thm:stochastic simple} there exist  environments where
\begin{align*}
    \Reg(T,\Pi_2) = \Omega\left(\frac{T^\alpha}{C_TT^{\alpha-\frac{1}{2}}}\sqrt{T}\right)=\tilde\Omega\left(T\right)\,.
\end{align*}
\end{proof}

Finally, we disprove the open problem in the stochastic case for any algorithm.
\begin{theorem}
\label{thm: improper}
No algorithm (proper or improper) can satisfy the requirements of open problem \eqref{eq:open} for all stochastic environments.
\end{theorem}
We present the high level proof ideas in the following subsections and the detailed proof in \cref{app:lower}.
\subsection{Adaptive adversary: $S$-switch($\Delta)$ Problem}
We present the adaptive environment in which model selection fails and the proof of \cref{thm:adaptive simple}.

The adversary switches the reward distribution up to $S$ many times, thereby segmenting the time into $S+1$ phases $(1,\dots,\tau_1,\tau_1+1,\dots,\tau_2,\dots,\tau_S,\dots T)$.
We denote $x_t\in[S+1]$ as the counter of phases and assume the agent is given this information.
For each phase $x_t\in[S]$, the adversary selects an optimal arm $(a^*_{s})_{s=1}^S$ uniformly at random among the first $K-1$ arms.
If $x_t\leq S$, the losses are i.i.d.\ Bernoulli random variables with means
\begin{align*}
    \E[\ell_{t,i}] = \frac{1}{2}-\begin{cases}
    0 & \mbox{ for }i\in[K-1]\setminus\{a^*_{x_t}\}\\
    \Delta &\mbox{ for }i=a^*_{x_t}\\
    \frac{7}{8}\Delta &\mbox{ for }i=K\,.
    \end{cases}
\end{align*}
In phase $S+1$, all losses are $0$ until the end of the game.
The adversary decides on the switching points based on an adaptive strategy. A switch from phase $s<S+1$ to $s+1$ occurs when the player has played $\Nmax=\lceil\frac{K-1}{192\Delta^2}\rceil$ times an arm in $[K-1]$ in phase $s$.
We can see this problem either as a special case of S-switch problem, or alternatively as a contextual bandit problem with $|\Pi_S|=K^{S+1}$ policies.

The lower bound proof for $S$-switch($\Delta$) relies on the following Lemma, which is proven in \cref{app:lower}.
\begin{lemma}
    \label{lem:basic}
    Let an agent interact with a $K-1\geq 2$ armed bandit problem with centered Bernoulli losses and randomized best arm of gap 
    $\Delta \leq \frac{1}{8\sqrt{3}}$
    for an adaptive number of time steps $N$. If the probability of 
    $N\geq \Nmax = \lceil\frac{K-1}{192\Delta^2}\rceil$
    is at least $\frac{1}{2}$, then the regret after $\Nmax$ time-steps conditioned on the event $N\geq \Nmax$ is lower bounded by
    \begin{align*}
        \Reg \geq \frac{\Delta}{4} \Nmax\,.
    \end{align*}
\end{lemma}
Informally, this Lemma says that conditioned on transitioning from phase $s$ to phase $s+1$, the agent has suffered regret $\Omega(\Delta\Nmax)$ against arm $K$ during phase $s$.
\begin{proof}[Informal proof of \cref{thm:adaptive simple}]  The adversary's strategy is designed in a way such that at each phase $s \in [S]$ it only allows the player's strategy to interact with the environment just enough times to discover the best action $a^*_{s}$. Then a new phase begins to prevent the player from exploiting knowledge of $a^*_s$. This ensures by \cref{lem:basic} that the player suffers regret at least $\Omega(\Delta N_{\max})$ during each completed phase. If an agent proceeds finding $a^*_s$ for all phases $s\in[S]$, then the regret against the non-switching baseline is $\smash{\Reg_{SW}(T,1)=O(\Delta\Nmax S)}$. By the assumption on the maximum regret of $\smash{\Reg_{SW}(T,1)}$ and an appropriate choice of $N_{\max}$ and $\Delta$, we can ensure that the agent must fail to discover all $a^*_s$ with constant probability, thus incurs regret at least $\smash{\Reg_{SW}(T,S)=\Omega(\Delta T)}$ against the optimal $S$-switch baseline. Tuning $\Delta$ and $\Nmax$ yield the desired theorem.
The formal argument with explicit choice of $\Delta$ is found in \cref{app:lower}.
\end{proof}

\subsection{Stochastic lower bound}
\label{sec:stocastic lower bound}
We now present the stochastic environment used for the impossibility results in \cref{thm:stochastic simple,thm: improper}.

There are $k+1$ environments $(\cE_i)_{i=0}^k$ with $\Pi_2=\{\pi_i|i\in[k]\cup\{0\}\}$ policies and $\Pi_1=\{\pi_0\}$.
In all environments, we have $K=3$ and $\pi_0$ always chooses action $3$, while $(\pi_i)_{i=1}^k$ are playing an action from $\{1,2\}$ uniformly at random. (In other words, the context is $\cX=\{1,2\}^k$ with $x_t$ sampled uniformly at random and $\pi_i(x)=x_i$.)

In each environment, the losses of actions $\{1,2\}$ at any time step satisfy $\ell_{t,1}=1-\ell_{t,2}$, which are conditioned on $x_t$ independent Bernoulli random variables,  with mean
\begin{align*}
    \E_{\cE_0}[\ell_{t,1}]=\E_{\cE_0}[\ell_{t,2}]=\frac{1}{2}\,\qquad\mbox{and }\forall i\in[k]:\,\E_{\cE_i}[\ell_{t,\pi_i(x_t)}]=\frac{1}{2}(1-\Delta)\,.
\end{align*}
Action $3$ gives a constant loss of $\frac{1}{2}-\frac{1}{4}\Delta
$ in all environments.

Let us unwrap these definitions. 
Playing either action 1 or action 2, which we call \emph{revealing} actions, yields  \emph{full-information} of all random variables at time $t$ due to the dependence of $\ell_{t,1}=1-\ell_{t,2}$ and the non-randomness of $\ell_{t,3}$. On the other hand, playing action 3 allows only to observe $x_t$, which has the same distribution in all environments, hence there is no information gained at all.

We know from full-information lower bounds that for optimal tuning of the gap, one suffers $\Omega(\sqrt{\ln(k)T})$ regret in the policy class $\Pi_2$, due to the difficulty of identifying the optimal arm.
For a smaller regret in policy class $\Pi_1$, one needs to confirm or reject the hypothesis $\cE_0$ faster than it takes to identify the optimal arm.
Existing techniques do not answer the question whether this is possible, and our main contribution of this section is to show that the hardness of rejecting $\cE_0$ is of the same order as identifying the exact environment.

For the remaining section, it will be useful  to consider a reparametrization of the random variables. Let $z_{t}\in\{0,1\}^k$ be the losses incurred by the policies $(\pi_i)_{i=1}^k$:  $z_{t,i}=\ell_{t,\pi_i(x_t)}$.
We can easily see that $z_t$ together with $x_{t,1}$ is sufficient to uniquely determine $\ell_t$ and $x_t$. Furthermore, $z_t$ is always a vector of independent Bernoulli random variables, which are independent of $x_{t,1}$\footnote{
We want to emphasize that $z_t$ is only independent of $x_{t,1}$, not independent of the full vector $x_t$.
}.
In environments $(\cE_i)_{i=1}^k$, the $i$-th component is a biased Bernoulli, while all other components have mean $\frac{1}{2}$. In $\cE_0$, no component is biased.
As before, $x_{t,1}$ does not provide any information since its distribution conditioned on $z_t$ is identical in all environments (see \cref{lem: indep lem} in \cref{app:lower} for a formal proof).

Under this reparameterization and ignoring non-informative bits of randomness, the problem of distinguishing $\cE_0$ from $\{\cE_i\}_{i=1}^k$ now looks as follows.
For time steps $t=1,\dots,T$, decide whether to play a revealing action and observe $z_t$ (potentially by taking $x_t$ into account). Use observed $(z_{\tau_n})_{n=1}^N$ to distinguish between the environments.
\emph{Proper} algorithms simplify the problem even further, because selecting $\pi_{i_t}$ independently of $x_t$ implies that the decision of observing $z_t$ is also independent of $x_t$ (any policy except $\pi_0$ allows to observe $z_t$ under any context). Hence for proper algorithms, we can reason directly about how many samples $z_t$ are required to distinguish between environments. This problem bears similarity to the property testing of dictator functions~\citep{balcan2012active} and sparse linear regression~\citep{ingster2010detection}\footnote{The setting of \citep{ingster2010detection} is different from our setting as they consider an asymptotic regime where both feature sparsity and dimensionality of the problem go to infinity, while for us the sparsity is fixed to one.}, however, there is no clear way to apply such results to our setting.

The following lemma shows the difficulty of testing for hypothesis $\cE_0$.
\begin{lemma}
\label{lem:sto lower}
Let $\Delta \leq \frac{1}{4}$, $k \geq e^{20}+1$, $N\leq\lfloor\frac{\ln(k-1)}{20\Delta^2}\rfloor$ and $\Delta^2N \geq \frac{1}{2}$. If the algorithm chooses whether to reveal $z_t$ independently of $x_t$ and if the total times $z_t$ is revealed is bounded by $N$ a.s.\, then for any measurable event $E$ it holds that
\begin{align*}
    \min_{i\in[k]}\bbP_{\cE_i}(E)-\bbP_{\cE_0}(E) \leq \frac{17}{\sqrt[4]{k-1}}\leq \frac{1}{4}\,.
\end{align*}
\end{lemma}
The proof of Lemma~\ref{lem:sto lower} is deferred to \cref{app:lower_bounts_proper_alg}. The high level idea is to directly bound the TV between $\min_{i\in[k]}\mathbb{P}_{\cE_i}$ and $\mathbb{P}_{\cE_0}$ over the space of outcomes of $(z_{t_n})_{n=1}^N$ by utilizing Berry-Essen's inequality instead of going through Pinsker's inequality. This step is key to achieve a dependence on $k$ in the bound.

For readers familiar with lower bound proofs for bandits and full-information, this Lemma should not come at a huge surprise.
For a $T$-round full-information game, it tells us that we can bias a single arm up to $\Delta = \Omega(\sqrt{\ln(k)/T})$, without this being detectable.
This directly recovers the well known lower bound of $\Delta T= \Omega(\sqrt{\ln(k)T})$ for full-information via the argument used for bandit lower bounds.
However, this result goes beyond what is known in the literature. We not only show that one cannot reliably detect the biased arm, but that one cannot even reliably detect whether any biased arm is present at all.
This property is the key to showing the lower bound of \cref{thm:stochastic simple}.

\begin{proof}[Informal proof of \cref{thm:stochastic simple}]
Under environment $\cE_0$, observing $z_t$ for $n$ time-steps incurs a regret of $\Reg_{\cE_0}(T,\Pi_1)=\Omega(\Delta n)$.
Using the assumption on the regret $\Reg_{\cE_0}(T,\Pi_1)$ and Markov inequality, we obtain an upper bound $N$ on the expected number of observations, which holds with probability $\frac{1}{2}$. 
We can construct an algorithm $\underline{\cA}$ that never observes more than $N$ samples, by following algorithm $\cA$ until it played $N$ times a revealing action and then commits to policy $\pi_0$ (action 3).
Since the algorithm $\underline{\cA}$ is proper, we can define $Z=(z_{\tau_i})_{i=1}^N$ as the observed $z$'s during time $\tau_i$ where the algorithm plays a revealing action.
For the revealed information generated by $\underline \cA$, we tune the remaining parameters such that the conditions of Lemma~\ref{lem:sto lower} are satisfied. Let $E$ be the event that $\underline{\cA}$ plays exactly $N$ times a revealing action (i.e. $\cA$ plays at least $N$ time the revealing action), then $E$ happens with probability $1-\Omega(1)$ under $\min_{i\in[k]} \mathbb{P}_{\cE_i}(E)$.
Thus, there exists an environment $i\in[k]$ such that $\cA$ plays less than $N$ times an action in $\{1,2\}$ with constant probability, which incurs regret of $\Reg_{\cE_i}(T,\Pi_2)=\Omega(\Delta T)$.
The theorem follows from tuning $\Delta$ and $N$, which is done formally in \cref{app:lower}.
\end{proof}

\paragraph{Improper algorithms.}
Even though we are not able to extend the lower bound proof uniformly over all values $\C$ and $k$ to improper algorithms, we can still show that no algorithm (proper or improper) can solve the open problem \eqref{eq:open} for stochastic environments.

The key is the following generalization of \cref{lem:sto lower}, which is proven in the appendix.
\begin{lemma}
\label{lem:sto lower improper}
Let $\Delta \leq \frac{1}{4}$, $k \geq e^{20}+1$, $N\leq\lfloor\frac{\ln(k-1)}{20\Delta^2}\rfloor$ and $\Delta^2N \geq \frac{1}{2}$. If the total number of times $z_t$ is revealed is bounded by $N$ a.s.\, then for any measurable event $E$ it holds that
\begin{align*}
    \min_{i\in[k]}\bbP_{\cE_i}(E)-\bbP_{\cE_0}(E) \leq \frac{17T^N}{\sqrt[4]{k-1}}\,.
\end{align*}
This holds even if the agent can take all contexts $(x_t)_{t=1}^T$ and previous observations into account when deciding whether to pick a revealing action at any time-step.
\end{lemma}

\begin{proof}[Informal proof of \cref{thm: improper}]
The proof is analogous to \cref{thm:stochastic simple}, however we use \cref{lem:sto lower improper} to bound the difference in probability of $E$ under $\cE_0$ and $\cE_i$. 
The key is to find a tuning such that the RHS of \cref{lem:sto lower improper} does not exceed $\frac{1}{4}$. Note that this is of order $\exp(\cO(\ln(T)N)-\Omega(\ln(k)) )$.
Let $\Delta = \Theta(1)$, then the requirement on $\Reg_{\cE_0}(T,\Pi_1)$ yields $N=\cO(T^\alpha)$. Setting $k=\Theta(\exp(T^{\alpha+\epsilon}))$ for any $\epsilon\in(0,1-\alpha)$,   then it follows immediately that the RHS in \cref{lem:sto lower improper} goes to $0$ for $T\rightarrow\infty$. Following the same arguments as in \cref{thm:stochastic simple}, there exists a sufficiently large $T$ up from which there always exists an environment $(\cE_i)_{i\in[k]}$ such that the regret is linear in $T$, thereby contradicting the open problem. The formal proof is deferred to \cref{app:lower}
\end{proof}

\section{Implications}
The relevance of open problem \cref{eq:open} has been motivated by its potential implications for other problems such as the S-switch bandit problem and an unresolved COLT2016 open problem on improved second order bounds for full-information.
Our negative result for \cref{eq:open} indeed lead to the expected insights.

\subsection{S-switch}
\label{sec:s-switch_comment}
Our lower bound in the \emph{adaptive} regime shows that adapting to the number of switches is hopeless if the timing of the switch points is not independent of the players actions.
Any algorithm adaptive to the number of switches in the regime with \emph{oblivious} adversary must break in the \emph{adaptive} case, which rules out bandit over bandit approaches based on importance sampling \citep{agarwal2016corralling}.
The successful algorithm proposed in \citet{cheung2019learning} is using a bandit over bandit approach without importance sampling. Nonetheless, all components have adaptive adversarial guarantees. 
The algorithm splits the time horizon into equal intervals of length $L$.
It initializes EXP3 with $\ln(T)$ arms, corresponding to a grid of learning rates.
For each epoch, the EXP3 top algorithm samples an arm and starts a freshly initialized instance of EXP3.S using the learning rate corresponding to the selected arm. This instance is run over the full epoch of length $L$. It collects the accumulated losses $L_{sum}=\sum_{t=1}^L\ell_t$ of the algorithm and feeds the loss $L_{sum}/L$ to the EXP3 top algorithm.

If all algorithms in the protocol enjoy guarantees against adaptive adversaries, why do bandit over bandit break against adaptive adversaries?
Adaptive adversaries are assumed to pick the losses $\ell_t$ independent of the choice of arm $A_t$ of the agent at round $t$. In the bandit over bandit protocol, the loss of the arm of the top algorithm dependents on the losses that the selected base suffers in the epoch. An adaptive adversary can adapt the losses in the epoch based on the actions of the base algorithm, that means the loss $\ell_t$ is not chosen independent of the action $A_t$. Hence the protocol is broken and the adaptive adversarial regret bounds do not hold.

\subsection{Second order bounds for full information.}
In an unresolved COLT2016 open problem, \citet{F16open} asks if it is possible to ensure a regret bound of order
\begin{align}
\label{eq:yoav}
    \Reg_{\varepsilon}=\tilde\cO\left(\sqrt{\sum_{t=1}^T\bbV_{i\sim p_t}[\ell_{t,i}] \ln(\frac{1}{\varepsilon})}\right)\,,
\end{align}
against the best $\varepsilon$ proportion of policies simultaneously for all $\varepsilon$. 
We go even a step further and show that the lower bound construction from \cref{sec:stocastic lower bound} directly provides a negative answer for any $\alpha<1$ to the looser bound
\begin{align}
\label{eq:yoav harder}
    \Reg_{\varepsilon}=\tilde\cO\left(\sqrt{\sum_{t=1}^T\left(\bbV_{i\sim p_t}[\ell_{t,i}]+T^{\alpha}\right) \ln(\frac{1}{\varepsilon})}\right)\,.
\end{align}

\begin{theorem}
\label{thm:yoav contradiction}
An algorithm satisfying \cref{eq:yoav harder} for $\alpha<1$ implies the existence of a proper algorithm that violates the lower bound for the counter example in the proof of \cref{thm:stochastic simple}. 
\end{theorem}
Theorem~\ref{thm:yoav contradiction} has the following interpretation. For any fixed $\alpha \in[0,1)$, there is no algorithm which enjoys a regret upper bound as in Equation~\ref{eq:yoav} for all problem instances s.t. $\sum_{t=1}^T\left(\bbV_{i\sim p_t}[\ell_{t,i}]\right) = \Theta(T^{\alpha})$. This implies we can not hope for a polynomial improvement, in terms of time horizon, over the existing bound of $\tilde O(\sqrt{T\ln(1/\epsilon)})$.
The detailed proof is found in \cref{app:lower}.
The high level idea is to initialize the full-information algorithm satisfying \cref{eq:yoav harder} with a sufficient number of copies of the baseline policy $\pi_0$ and to feed importance weighted losses of the experts (i.e. policies) to that algorithm.

As we mention in Section~\ref{sec:intro}, the case $\alpha=1$ is obtainable. 
Our reduction relates the adaptation to variance to the model selection problem. As in \cref{eq:open}, $\alpha$ is the trade-off between time and complexity. An algorithm satisfying \cref{eq:yoav harder} with $\alpha=1$ merely allows to recover the trivial $\cO(T)$ bound for model selection, and hence does not lead to a contradiction.

\section{Conclusion}
We derived the Pareto Frontier of minimax regret for model selection in Contextual bandits.
Our results have resolved several open problems \citep{FKL20,F16open}.

\begin{ack}
We like to thank Haipeng Luo and Yoav Freund for discussions about our lower bound proofs. We thank Tor Lattimore for pointing us to the technicalities required for bounding the total variation of improper algorithms.
\end{ack}


\bibliographystyle{plainnat}
\bibliography{mybib}

\begin{thebibliography}{35}
\providecommand{\natexlab}[1]{#1}
\providecommand{\url}[1]{\texttt{#1}}
\expandafter\ifx\csname urlstyle\endcsname\relax
  \providecommand{\doi}[1]{doi: #1}\else
  \providecommand{\doi}{doi: \begingroup \urlstyle{rm}\Url}\fi

\bibitem[Abbasi-Yadkori et~al.(2011)Abbasi-Yadkori, P{\'a}l, and
  Szepesv{\'a}ri]{abbasi2011improved}
Yasin Abbasi-Yadkori, D{\'a}vid P{\'a}l, and Csaba Szepesv{\'a}ri.
\newblock Improved algorithms for linear stochastic bandits.
\newblock In \emph{NIPS}, volume~11, pages 2312--2320, 2011.

\bibitem[Agarwal et~al.(2017)Agarwal, Luo, Neyshabur, and
  Schapire]{agarwal2016corralling}
Alekh Agarwal, Haipeng Luo, Behnam Neyshabur, and Robert~E Schapire.
\newblock Corralling a band of bandit algorithms.
\newblock 2017.

\bibitem[Arora et~al.(2021)Arora, Marinov, and Mohri]{arora2021corralling}
Raman Arora, Teodor~Vanislavov Marinov, and Mehryar Mohri.
\newblock Corralling stochastic bandit algorithms.
\newblock In \emph{International Conference on Artificial Intelligence and
  Statistics}, pages 2116--2124. PMLR, 2021.

\bibitem[Auer(2002)]{auer2002using}
Peter Auer.
\newblock Using confidence bounds for exploitation-exploration trade-offs.
\newblock \emph{Journal of Machine Learning Research}, 3\penalty0
  (Nov):\penalty0 397--422, 2002.

\bibitem[Auer et~al.(2002)Auer, Cesa-Bianchi, Freund, and
  Schapire]{auer2002nonstochastic}
Peter Auer, Nicol{\`o} Cesa-Bianchi, Yoav Freund, and Robert~E Schapire.
\newblock The nonstochastic multiarmed bandit problem.
\newblock \emph{SIAM journal on computing}, 32\penalty0 (1), 2002.

\bibitem[Auer et~al.(2019)Auer, Gajane, and Ortner]{auer2019adaptively}
Peter Auer, Pratik Gajane, and Ronald Ortner.
\newblock Adaptively tracking the best bandit arm with an unknown number of
  distribution changes.
\newblock In \emph{Conference on Learning Theory}, pages 138--158. PMLR, 2019.

\bibitem[Balcan et~al.(2012)Balcan, Blais, Blum, and Yang]{balcan2012active}
Maria-Florina Balcan, Eric Blais, Avrim Blum, and Liu Yang.
\newblock Active property testing.
\newblock In \emph{2012 IEEE 53rd Annual Symposium on Foundations of Computer
  Science}, pages 21--30. IEEE, 2012.

\bibitem[Besbes et~al.(2014)Besbes, Gur, and Zeevi]{besbes2014stochastic}
Omar Besbes, Yonatan Gur, and Assaf Zeevi.
\newblock Stochastic multi-armed-bandit problem with non-stationary rewards.
\newblock \emph{Advances in neural information processing systems},
  27:\penalty0 199--207, 2014.

\bibitem[Besbes et~al.(2015)Besbes, Gur, and Zeevi]{besbes2015non}
Omar Besbes, Yonatan Gur, and Assaf Zeevi.
\newblock Non-stationary stochastic optimization.
\newblock \emph{Operations research}, 63\penalty0 (5):\penalty0 1227--1244,
  2015.

\bibitem[Carl-Gustav(1942)]{carl1942liapunoff}
Esseen Carl-Gustav.
\newblock On the liapunoff limit of error in the theory of probability.
\newblock \emph{Arkiv for matematik, astronomi och fysik, A: 1--19}, 1942.

\bibitem[Cesa-Bianchi et~al.(2007)Cesa-Bianchi, Mansour, and
  Stoltz]{cesa2007improved}
Nicolo Cesa-Bianchi, Yishay Mansour, and Gilles Stoltz.
\newblock Improved second-order bounds for prediction with expert advice.
\newblock \emph{Machine Learning}, 66\penalty0 (2):\penalty0 321--352, 2007.

\bibitem[Chatterji et~al.(2020)Chatterji, Muthukumar, and
  Bartlett]{chatterji2020osom}
Niladri Chatterji, Vidya Muthukumar, and Peter Bartlett.
\newblock Osom: A simultaneously optimal algorithm for multi-armed and linear
  contextual bandits.
\newblock In \emph{International Conference on Artificial Intelligence and
  Statistics}, pages 1844--1854. PMLR, 2020.

\bibitem[Chaudhuri et~al.(2009)Chaudhuri, Freund, and
  Hsu]{chaudhuri2009parameter}
Kamalika Chaudhuri, Yoav Freund, and Daniel Hsu.
\newblock A parameter-free hedging algorithm.
\newblock \emph{arXiv preprint arXiv:0903.2851}, 2009.

\bibitem[Chernov and Vovk(2010)]{chernov2010prediction}
Alexey Chernov and Vladimir Vovk.
\newblock Prediction with advice of unknown number of experts.
\newblock \emph{arXiv preprint arXiv:1006.0475}, 2010.

\bibitem[Cheung et~al.(2019)Cheung, Simchi-Levi, and Zhu]{cheung2019learning}
Wang~Chi Cheung, David Simchi-Levi, and Ruihao Zhu.
\newblock Learning to optimize under non-stationarity.
\newblock In \emph{The 22nd International Conference on Artificial Intelligence
  and Statistics}, pages 1079--1087. PMLR, 2019.

\bibitem[Chu et~al.(2011)Chu, Li, Reyzin, and Schapire]{chu2011contextual}
Wei Chu, Lihong Li, Lev Reyzin, and Robert Schapire.
\newblock Contextual bandits with linear payoff functions.
\newblock In \emph{Proceedings of the Fourteenth International Conference on
  Artificial Intelligence and Statistics}, pages 208--214. JMLR Workshop and
  Conference Proceedings, 2011.

\bibitem[Cutkosky et~al.(2020)Cutkosky, Das, and Purohit]{cutkosky2020upper}
Ashok Cutkosky, Abhimanyu Das, and Manish Purohit.
\newblock Upper confidence bounds for combining stochastic bandits.
\newblock \emph{arXiv preprint arXiv:2012.13115}, 2020.

\bibitem[Foster et~al.(2019)Foster, Krishnamurthy, and Luo]{foster2019model}
Dylan~J Foster, Akshay Krishnamurthy, and Haipeng Luo.
\newblock Model selection for contextual bandits.
\newblock \emph{arXiv preprint arXiv:1906.00531}, 2019.

\bibitem[Foster et~al.(2020{\natexlab{a}})Foster, Gentile, Mohri, and
  Zimmert]{FGMZ20}
Dylan~J Foster, Claudio Gentile, Mehryar Mohri, and Julian Zimmert.
\newblock Adapting to misspecification in contextual bandits.
\newblock \emph{Advances in Neural Information Processing Systems}, 33,
  2020{\natexlab{a}}.

\bibitem[Foster et~al.(2020{\natexlab{b}})Foster, Krishnamurthy, and
  Luo]{FKL20}
Dylan~J Foster, Akshay Krishnamurthy, and Haipeng Luo.
\newblock Open problem: Model selection for contextual bandits.
\newblock In \emph{Conference on Learning Theory}, pages 3842--3846. PMLR,
  2020{\natexlab{b}}.

\bibitem[Freund(2016)]{F16open}
Yoav Freund.
\newblock Open problem: Second order regret bounds based on scaling time.
\newblock In \emph{Conference on Learning Theory}, pages 1651--1654. PMLR,
  2016.

\bibitem[Ghosh et~al.(2021)Ghosh, Sankararaman, and Kannan]{ghosh2021problem}
Avishek Ghosh, Abishek Sankararaman, and Ramchandran Kannan.
\newblock Problem-complexity adaptive model selection for stochastic linear
  bandits.
\newblock In \emph{International Conference on Artificial Intelligence and
  Statistics}, pages 1396--1404. PMLR, 2021.

\bibitem[Ingster et~al.(2010)Ingster, Tsybakov, and
  Verzelen]{ingster2010detection}
Yuri~I Ingster, Alexandre~B Tsybakov, and Nicolas Verzelen.
\newblock Detection boundary in sparse regression.
\newblock \emph{Electronic Journal of Statistics}, 4:\penalty0 1476--1526,
  2010.

\bibitem[Koolen and Van~Erven(2015)]{koolen2015second}
Wouter~M Koolen and Tim Van~Erven.
\newblock Second-order quantile methods for experts and combinatorial games.
\newblock In \emph{Conference on Learning Theory}, pages 1155--1175. PMLR,
  2015.

\bibitem[Langford and Zhang(2007)]{langford2007epoch}
John Langford and Tong Zhang.
\newblock The epoch-greedy algorithm for contextual multi-armed bandits.
\newblock \emph{Advances in neural information processing systems}, 20\penalty0
  (1):\penalty0 96--1, 2007.

\bibitem[Lattimore(2015)]{lattimore2015pareto}
Tor Lattimore.
\newblock The pareto regret frontier for bandits.
\newblock \emph{arXiv preprint arXiv:1511.00048}, 2015.

\bibitem[Luo and Schapire(2015)]{luo2015achieving}
Haipeng Luo and Robert~E Schapire.
\newblock Achieving all with no parameters: Adaptive normalhedge.
\newblock \emph{arXiv preprint arXiv:1502.05934}, 2015.

\bibitem[Luo et~al.(2018)Luo, Wei, Agarwal, and Langford]{luo2018efficient}
Haipeng Luo, Chen-Yu Wei, Alekh Agarwal, and John Langford.
\newblock Efficient contextual bandits in non-stationary worlds.
\newblock In \emph{Conference On Learning Theory}, pages 1739--1776. PMLR,
  2018.

\bibitem[Pacchiano et~al.(2020{\natexlab{a}})Pacchiano, Dann, Gentile, and
  Bartlett]{pacchiano2020regret}
Aldo Pacchiano, Christoph Dann, Claudio Gentile, and Peter Bartlett.
\newblock Regret bound balancing and elimination for model selection in bandits
  and rl.
\newblock \emph{arXiv preprint arXiv:2012.13045}, 2020{\natexlab{a}}.

\bibitem[Pacchiano et~al.(2020{\natexlab{b}})Pacchiano, Phan, Abbasi-Yadkori,
  Rao, Zimmert, Lattimore, and Szepesvari]{pacchiano2020model}
Aldo Pacchiano, My~Phan, Yasin Abbasi-Yadkori, Anup Rao, Julian Zimmert, Tor
  Lattimore, and Csaba Szepesvari.
\newblock Model selection in contextual stochastic bandit problems.
\newblock \emph{arXiv preprint arXiv:2003.01704}, 2020{\natexlab{b}}.

\bibitem[Tyurin(2009)]{tyurin2009new}
Ilya Tyurin.
\newblock New estimates of the convergence rate in the lyapunov theorem.
\newblock \emph{arXiv preprint arXiv:0912.0726}, 2009.

\bibitem[Wei et~al.(2017)Wei, Hong, and Lu]{wei2017tracking}
Chen-Yu Wei, Yi-Te Hong, and Chi-Jen Lu.
\newblock Tracking the best expert in non-stationary stochastic environments.
\newblock \emph{arXiv preprint arXiv:1712.00578}, 2017.

\bibitem[Zhao et~al.(2020)Zhao, Zhang, Jiang, and Zhou]{zhao2020simple}
Peng Zhao, Lijun Zhang, Yuan Jiang, and Zhi-Hua Zhou.
\newblock A simple approach for non-stationary linear bandits.
\newblock In \emph{International Conference on Artificial Intelligence and
  Statistics}, pages 746--755. PMLR, 2020.

\bibitem[Zhu and Nowak(2021)]{zhu2021pareto}
Yinglun Zhu and Robert Nowak.
\newblock Pareto optimal model selection in linear bandits.
\newblock \emph{arXiv preprint arXiv:2102.06593}, 2021.

\bibitem[Zimmert and Seldin(2021)]{zimmert2021tsallis}
Julian Zimmert and Yevgeny Seldin.
\newblock Tsallis-inf: An optimal algorithm for stochastic and adversarial
  bandits.
\newblock \emph{Journal of Machine Learning Research}, 22\penalty0
  (28):\penalty0 1--49, 2021.

\end{thebibliography}
\newpage

\appendix

\section*{Checklist}

\begin{enumerate}
\item For all authors...
\begin{enumerate}
  \item Do the main claims made in the abstract and introduction accurately reflect the paper's contributions and scope?
    \answerYes{Our main regret upper bounds are found in Theorem~\ref{thm:upper main} and Theorem~\ref{thm:upper general}. Our main lower bounds are found in Theorem~\ref{thm:adaptive simple} and Theorem~\ref{thm:stochastic simple} with Corollary~\ref{cor:main} solving the COLT2020 open problem and Theorem~\ref{thm:yoav contradiction} solving the COLT2016 open problem.}
  \item Did you describe the limitations of your work?
    \answerYes{See paragraph regarding removing the properness requirement at the end of page 8.}
  \item Did you discuss any potential negative societal impacts of your work?
    \answerNA{This paper is theoretical in nature and we do not foresee any immediate societal impacts.}
  \item Have you read the ethics review guidelines and ensured that your paper conforms to them?
    \answerYes{}
\end{enumerate}

\item If you are including theoretical results...
\begin{enumerate}
  \item Did you state the full set of assumptions of all theoretical results?
    \answerYes{}
	\item Did you include complete proofs of all theoretical results?
    \answerYes{For complete proofs we refer the reader to the appendix.}
\end{enumerate}

\item If you ran experiments...
\begin{enumerate}
  \item Did you include the code, data, and instructions needed to reproduce the main experimental results (either in the supplemental material or as a URL)?
    \answerNA{}
  \item Did you specify all the training details (e.g., data splits, hyperparameters, how they were chosen)?
    \answerNA{}
	\item Did you report error bars (e.g., with respect to the random seed after running experiments multiple times)?
    \answerNA{}
	\item Did you include the total amount of compute and the type of resources used (e.g., type of GPUs, internal cluster, or cloud provider)?
    \answerNA{}
\end{enumerate}

\item If you are using existing assets (e.g., code, data, models) or curating/releasing new assets...
\begin{enumerate}
  \item If your work uses existing assets, did you cite the creators?
    \answerNA{}
  \item Did you mention the license of the assets?
    \answerNA{}
  \item Did you include any new assets either in the supplemental material or as a URL?
    \answerNA{}
  \item Did you discuss whether and how consent was obtained from people whose data you're using/curating?
    \answerNA{}
  \item Did you discuss whether the data you are using/curating contains personally identifiable information or offensive content?
    \answerNA{}
\end{enumerate}

\item If you used crowdsourcing or conducted research with human subjects...
\begin{enumerate}
  \item Did you include the full text of instructions given to participants and screenshots, if applicable?
    \answerNA{}
  \item Did you describe any potential participant risks, with links to Institutional Review Board (IRB) approvals, if applicable?
    \answerNA{}
  \item Did you include the estimated hourly wage paid to participants and the total amount spent on participant compensation?
    \answerNA{}
\end{enumerate}

\end{enumerate}

\section{Upper bound proofs}
\label{app:upper}
\begin{proof}[Proof of \cref{thm:upper general}]
Denote with $m_t$ the arm that algorithm $m$ would choose if it would be selected during round $t$. We decompose the regret into
\begin{align*}
    \Reg(T,\Pi_m) &= \max_{\pi \in \Pi_m}\E\left[\sum_{t=1}^T \ell_{t,A_t} - \ell_{t,\pi(x_t)}\right]\\
    &=\max_{\pi \in \Pi_m}\E\left[\sum_{t=1}^T \ell_{t,A_t} -\ell_{t,m_t}+\ell_{t,m_t}- \ell_{t,\pi(x_t)}\right]\\
    &=\E\left[\sum_{t=1}^T \ell_{t,A_t} -\ell_{t,m_t}\right]+\max_{\pi \in \Pi_m}\E\left[\sum_{t=1}^T\frac{\bbI(M_t=m)}{q_{t,m}}(\ell_{t,m_t}- \ell_{t,\pi(x_t)})\right]\\
    &\leq\E\left[\sum_{t=1}^T \ell_{t,A_t} -\ell_{t,m_t}\right]+\E[\sqrt{\rho_{T,m}}]\sqrt{\C_m T}\,,
\end{align*}
where the last line is by the assumption of the theorem.
The first term requires some basic properties of FTRL analysis, see e.g. \citep{zimmert2021tsallis}. 
Define $\hat L_t=\sum_{s=1}^t\hat\ell_t$, $\tilde B_t=\sum_{s=1}^tb_t$. 
For Tsallis-INF with constant learning rate\footnote{The proof can be adapted to time dependent learning rates.} $\eta = \frac{1}{\sqrt{T}}$, we have the following properties
\begin{align*}
    F(q) &= -2\sum_{i=1}^M \sqrt{q_i}\\
    \bar F^*(-L) &= \max_{q\in\Delta([M])} \ip{q,-L}-\eta^{-1}F(q)\\
    q_t &= \nabla \bar F^*(-(\hat L_{t-1}-\tilde B_{t-1}))\,.
\end{align*}
The standard FTRL proof (e.g.~\citet{zimmert2021tsallis}) shows that
\begin{align*}
    &\forall t: \E_{M_t\sim q_t}[D_{\bar F^*}(-(\hat L_t-\tilde B_{t-1}),-(\hat L_{t-1}-\tilde B_{t-1}) )]\leq \eta\sqrt{K}.
\end{align*}

The first term is
\begin{align*}
    \E\left[\sum_{t=1}^T \ell_{t,A_t} -\ell_{t,m_t}\right] &= 
    \E\left[\sum_{t=1}^T \ip{q_t-\mathbf{e}_{m}, \hat\ell_t}\right]\\
    &= 
    \E\bigg[\sum_{t=1}^T \bar F^*(-(\hat L_{t-1}-\tilde B_{t-1}))-\bar F^*(-(\hat L_{t}-\tilde B_{t-1}))\\
    &+D_{\bar F^*}(-(\hat L_t-\tilde B_{t-1}),-(\hat L_{t-1}-\tilde B_{t-1}) )\bigg]-\E\left[\hat L_{T,m}\right].
\end{align*}
Let us consider the terms $\bar F^*(-(\hat L_{t-1}-\tilde B_{t-1}))$ and $\bar F^*(-(\hat L_{t}-\tilde B_{t-1}))$. First, using the definition of the conjugate function and $q_{t}$ we know that 
\begin{align*}
    \bar F^*(-(\hat L_{t-1} - \tilde B_{t-1})) + \bar F(q_t) = \langle q_t, -(\hat L_{t-1} - \tilde B_{t-1})\rangle.
\end{align*}
Further by Young's inequality it holds that 
\begin{align*}
    \bar F^*(-(\hat L_{t}-\tilde B_{t-1})) + \bar F(q_{t+1}) &\geq \langle q_{t+1},-(\hat L_{t}-\tilde B_{t-1}) \rangle \implies\\
    -\bar F^*(-(\hat L_{t}-\tilde B_{t-1})) &\leq \bar F(q_{t+1}) + \langle q_{t+1},\hat L_{t}-\tilde B_{t-1} \rangle
\end{align*}
The above two displays imply
\begin{align*}
     \bar F^*(-(\hat L_{t} - \tilde B_{t})) - \bar F^*(-(\hat L_{t}-\tilde B_{t-1})) &\leq \bar F(q_{t+1}) + \langle q_{t+1},\hat L_{t}-\tilde B_{t-1} \rangle\\
     &\qquad- \bar F(q_{t+1}) - \langle q_{t+1}, \hat L_{t} - \tilde B_{t}\rangle\\
     &= \langle q_{t+1}, b_t \rangle.
\end{align*}
Thus we can bound 
\begin{align*}
    \sum_{t=1}^T \bar F^*(-(\hat L_{t-1}-\tilde B_{t-1}))-\bar F^*(-(\hat L_{t}-\tilde B_{t-1})) &\leq \bar F^*(0) - \bar F^*(-(\hat L_T - B_{T-1})) + \sum_{t=1}^{T-1} \langle q_{t+1}, b_t \rangle\\
    &\leq \bar F^*(0) - \bar F(\mathbf{e}_{m}) + \langle \mathbf{e}_{m}, \hat L_T \rangle - \langle \mathbf{e}_{m}, \tilde B_{T-1} \rangle\\
    &\qquad+\sum_{t=1}^{T-1} \langle q_{t+1}, b_t \rangle\\
    &\leq \bar F^*(0) - \bar F(\mathbf{e}_{m}) + \langle \mathbf{e}_{m}, \hat L_T \rangle - \sqrt{\rho_{T,m}}R_m\\
    &\qquad+\sqrt{\rho_{1,m}}R_m + \sum_{t=1}^{T-1} \langle q_{t+1}, b_t \rangle\\
    &\leq \frac{\sqrt{M}}{\eta} + \langle \mathbf{e}_{m}, \hat L_T \rangle - \sqrt{\rho_{T,m}}R_m\\
    &\qquad+\sqrt{\rho_{1,m}}R_m + \sum_{t=1}^{T-1} \langle q_{t+1}, b_t \rangle
    .
\end{align*}
Taking expectation and setting $\eta$ appropriately we have that
\begin{align*}
    \E\left[\sum_{t=1}^T \ell_{t,A_t} -\ell_{t,m_t}\right] \leq 2\sqrt{2MT}+\E\left[\sum_{t=1}^T \ip{q_{t+1}, b_t}\right]-\left(\E[\sqrt{\rho_{T,m}}]-\frac{\sqrt{M\C_m}}{\C}\right)\sqrt{\C_m T}.
\end{align*}

Finally the final term is
\begin{align*}
    \sum_{t=1}^T \ip{q_{t+1}, b_t} &= \sum_{i=1}^M\sum_{t=1}^T \rho_{t+1,i}^{-1}(\sqrt{\rho_{t+1,i}}-\sqrt{\rho_{t,i}})R_i\\
    &\leq \sum_{i=1}^M\int_{\sqrt{\beta_i}}^\infty q^{-1}\,dq\,R_i\\
    &= \sum_{i=1}^M \frac{R_i}{\sqrt{\beta_i}}=\C\sqrt{MT}\,.
\end{align*}
Putting all the bounds together finishes the proof.
\end{proof}
\section{Lower bounds}
\label{app:lower}

\subsection{Adaptive adversary lower bound and proof of Theorem~\ref{thm:adaptive simple}}

\begin{proof}[Proof of \cref{lem:basic}]
    We denote two environments $\cE_0,\cE_1$. In both environments, we define the random variable $A^*\in[K-1]$ chosen uniformly at random and obliviously to the agents.
    In the environment $\cE_0$, it is not possible to observe any information about $A^*$ because the losses at any time step are sampled i.i.d.\  $\ell_{t}\sim\Ber(\frac{1}{2})$ independent of the action picked.
    In environment $\cE_1$, the loss is still $\ell_{t}\sim\Ber(\frac{1}{2})$ when $A_t\neq A^*$, but differs when $A_t=A^*$. In this case, it is instead drawn according to  $\ell_{t}\sim\Ber(\frac{1}{2}-\Delta)$.
    The agent might interact with the environment for less than $\Nmax$ time-steps, however we define the probability measures $\bbP_{\cE_0}$ and $\bbP_{\cE_1}$ according to exactly $\Nmax$ observations of the environment. If the agent stops playing at time $\tau<\Nmax$, we simply assume the environment continues playing random actions until the end of the game.
    
    Denote $Z=\bbI\{\tau\geq\Nmax\}$ as the indicator of whether the agent plays for sufficiently many time steps.
    The condition in the lemma reads $\bbP_{\cE_1}[Z]\geq \frac{1}{2}$.
    Since in environment $\cE_0$, the agent does not receive any information about $A^*$, it holds $\sum_{t=1}^{\Nmax} \bbP_{\cE_0}(A_t=A^*)=\Nmax /(K-1)$.
    Hence by the divergence decomposition rule of the KL divergence, we have
    \begin{align*}
        D_{KL}(\bbP_{\cE_0} || \bbP_{\cE_1})=\sum_{t=1}^{\Nmax} \bbP_{\cE_0}(A_t=A^*)\kl\left(\frac{1}{2},\frac{1}{2}-\Delta\right) \leq \frac{3\Delta^2\Nmax}{K-1}\leq \frac{1}{32},
    \end{align*}
    where $\kl(p,q)$ denotes the KL-divergence between two Bernoulli distributions with parameters $p$ and $q$ respectively and the last inequality follows as
    \begin{align*}
        \kl\left(\frac{1}{2},\frac{1}{2}-\Delta\right) = -\frac{1}{2}\ln(1-4\Delta^2) \leq 3\Delta^2.
    \end{align*}
    By chaining Pinsker's inequality, we have
    \begin{align*}
        \bbP_{\cE_1}(A_t=A^* \,\land\, Z) &\leq \bbP_{\cE_0}(A_t=A^* \,\land\, Z) + \sqrt{2D_{KL}(\bbP_{\cE_0} || \bbP_{\cE_1})}\\
        &= \frac{\bbP_{\cE_0}(Z)}{K-1} + \sqrt{2D_{KL}(\bbP_{\cE_0} || \bbP_{\cE_1})}\\
        &\leq \frac{\bbP_{\cE_1}(Z)}{K-1} + \left(1+\frac{1}{K-1}\right)\sqrt{2D_{KL}(\bbP_{\cE_0} || \bbP_{\cE_1})}\,. 
    \end{align*}
    Hence
    \begin{align*}
        \bbP_{\cE_1}(A_t=A^* \, \vert\,Z) \leq \frac{1}{K-1} + \frac{(1+1/(K-1))\sqrt{2D_{KL}(\bbP_{\cE_0} || \bbP_{\cE_1})}}{\bbP_{\cE_1}(Z)}\leq \frac{1}{K-1}+3\sqrt{\frac{1}{16}}= \frac{3}{4}\,.
    \end{align*}
    Finally, using the regret definition and combining everything
    \begin{align*}
        \Reg = \E_{\cE_1}\left[\sum_{t=1}^{\Nmax}\bbI\{A_t\neq A^*\}\Delta\,\vert\,Z\right]=\sum_{t=1}^{\Nmax}(1-\bbP(A_t=A^*\,\vert\,Z))\Delta\geq \frac{1}{4}\Delta\Nmax\,.
    \end{align*}
\end{proof}

\begin{proof}[Proof of \cref{thm:adaptive simple}]
Let $\cT=x_T-1$ be the number of switches the agent triggers from the adversary in the game.
We set
$\Delta = \min\left\{\frac{S\sqrt{K-1}}{3072\C\sqrt{T}},\frac{1}{8\sqrt{3}}\right\}$
, which guarantees 
\[\Nmax = \left\lceil\max\left\{\frac{49152\C^2T}{S^2},K-1\right\}\right\rceil\leq \frac{T}{2S}\,.\]
The regret is bounded by
\begin{align*}
    \Reg_{SW}(T,S) \geq \bbP(\cT \neq S)\frac{\Delta}{8}(T-S\Nmax)=(1-\bbP(\cT = S))\frac{\Delta T}{16}\,,
\end{align*}
since the agent cannot have played the optimal action more than $S\Nmax$ times without triggering $S$ switches.
If the probability of triggering the $S$'s switch is below $\frac{1}{2}$, we are done.
Otherwise by \cref{lem:basic}, the regret against a non-switching baseline on arm $K$ is bounded by
\begin{align*}
    \Reg_{SW}(T,1) \geq \bbP(\cT= S) \frac{\Delta}{8}\Nmax S\,.
\end{align*}
By assumption, we have $\Reg_{SW}(T,1)\leq \C\sqrt{(K-1)T}$, hence $\bbP(\cT= S)\leq \frac{8\C\sqrt{(K-1)T}}{\Delta\Nmax S}$.
Plugging this into the bound above, yields
\begin{align*}
    \Reg_{SW}(T,S) \geq \left(1-\frac{8\C\sqrt{(K-1)T}}{\Delta\Nmax S}\right)\frac{\Delta T}{16}\geq \frac{\Delta T}{32}=\Omega\left(\min\left\{\frac{S}{\C}\sqrt{KT},T\right\}\right)\,.
\end{align*}
\end{proof}

\subsection{Stochastic lower bound for proper algorithms and proof of Theorem~\ref{thm:stochastic simple}}
\label{app:lower_bounts_proper_alg}
To proof of our key lemma, Lemma~\ref{lem:sto lower} we first begin by showing that we can restrict our attention only to the outcome space of $(z_t)_{t=1}^N$, where $z_{t,i} = \ell_{t,\pi_i(x_t)}$. This done in the following lemma.

\begin{lemma}
\label{lem: indep lem}
For any $\cE_i, i\in[k]\cup\{0\}$
there exist a bijection from $(z_t = (\ell_{t,\pi_i(x_t)})_{i=1}^k, x_{t,1})$ to $(\ell_t, x_t)$.
$z_t$ is a collection of independent Bernoulli random variables. The means satisfy 
\[\E_{\cE_i}[z_{t,j}] = \begin{cases}\frac{1}{2}&\mbox{ if }i\neq j\\
\frac{1}{2}(1-\Delta)&\mbox{ otherwise.}
\end{cases}\]
Finally $x_{t,1}$ is independent of $z_t$.
\end{lemma}
\begin{proof}
The direction $(\ell_t,x_t)\rightarrow (z_t,x_{t,1})$ is given by the definition of $z_t$. The other direction is provided by
\begin{align*}
    \ell_{t,i} = \begin{cases} z_{t,1} &\mbox{ if }i=x_{t,1}\\
    1-z_{t,1} &\mbox{ otherwise,}
    \end{cases},\qquad
    x_{t,i} = \begin{cases} x_{t,1} &\mbox{ if }z_{t,i}=z_{t,1}\\
    3-x_{t,1} &\mbox{ otherwise.}
    \end{cases}
\end{align*}
To show the independence of $z_{t,i}$, note that by the data generation process we select $z_{t,i}=\ell_{t,\pi_i(x_t)}$ in environment $i$ such that it is a  Bernoulli with mean $\frac{1}{2}(1-\Delta)$ independent of $x_t$ (in environment $\cE_0$, we sample $z_{t,1}$ with mean $\frac{1}{2}$). Now conditioned on $z_{t,i}$, the value of $z_{t,j}, j\neq i$ depends on whether $x_{t,i}=x_{t,j}$. $x_{t}$ is a collection of i.i.d.\ Bernoulli's with mean $\frac{1}{2}$, hence $z_{t,j}$ is independent of $z_{t,i}$ with mean $\frac{1}{2}$. We continue with the same argument over all $k$ and show that all components of $z_t$ are independent with the claimed means.
Finally to show that $x_{t,1}$ is independent of $z_t$, observe that we have total symmetry over the arms $\{1,2\}$ in our construction. Hence for any environment $i$, we have that \[
\bbP_{\cE_i}[x_{t,1}=1|z_t]=\bbP_{\cE_i}[x_{t,1}=2|z_t]=\frac{1}{2}=\bbP_{\cE_i}[x_{t,1}=1]=\bbP_{\cE_i}[x_{t,1}=2]\,.\]
\end{proof}

We now recall the Berry-Essen inequality which we will use in the proof of Lemma~\ref{lem:sto lower}.
\begin{theorem}[Berry-Essen inequality~\citep{carl1942liapunoff,tyurin2009new}]
Let $Y_1,\ldots,Y_n$ be independent mean zero random variables second moment $\mathbb{E}[Y_i^2] = \sigma_i^2$ and absolute third moment $\mathbb{E}[|Y_i|^{3}] = \rho_i$. If $F_n$ denotes the CDF of $\frac{\sum_{i=1}^n Y_i}{\sqrt{\sum_{i=1}^n \sigma_i^2}}$ and $\Phi$ denotes the CDF of a standard Gaussian variable then it holds that
\begin{align*}
    \sup_{\alpha\in\mathbb{R}}|F_n(\alpha) - \Phi(\alpha)| \leq C\frac{\sum_{i=1}^n \rho_i}{\left(\sum_{i=1}^n \sigma_i^2\right)^{3/2}},
\end{align*}
where $C$ is some absolute constant upper bounded by $1$.
\end{theorem}

We will also need the following result relating the third central moment of a non-negative random variable to the third moment.
\begin{claim}
\label{claim:moment_ineq}
Let $X \geq 0$ be a non-negative r.v. with mean $\mu > 0$. Then $\mathbb{E}[|X-\mu|^{3}] \leq 2\mathbb{E}[X^3]$.
\end{claim}
\begin{proof}
From triangle inequality and Jensen's inequality we have
\begin{align*}
    \mathbb{E}[|X-\mu|^{3}] \leq \mathbb{E}[\max\{|X|^3,|\mu|^3\}] \leq \mathbb{E}[X^3]+\mathbb{E}[X]^3\leq 2\mathbb{E}[X^3],
\end{align*}
because $X$ and $\mu$ are non-negative.
\end{proof}

Finally, we need the following useful TV distance inequalities.
\begin{lemma}
\label{lem:tv_ineqs}
 For any random variables $X,Y$ (not necessarily independent) over measures $\bbP,\bbQ$ it holds that
 \begin{align*}
     &\|\bbP_{X,Y}-\bbQ_{X,Y}\|_{TV} \leq \E_{X\sim \bbP_X}\left[\|\bbP_{Y|X}-\bbQ_{Y|X}\|_{TV}\right]+\|\bbP_{X}-\bbQ_{X}\|_{TV}\,.
 \end{align*}
\end{lemma}
\begin{proof}
Denote $s_{x,y}=\operatorname{sign}(\bbP(X=x,Y=y)-\bbQ(X=x,Y=y))$.
We have
\begin{align*}
    \|\bbP_{X,Y}-\bbQ_{X,Y}\|_{TV} &= \sum_{x,y}  s_{x,y}(p(x)\cdot p(y|x)-q(x)\cdot q(y|x))\\
    &=\sum_{x} p(x) \sum_y s_{x,y}(p(y|x)-q(y|x))
    +\sum_{x,y} s_{x,y}(p(x)-q(x))q(y|x)\,.
\end{align*}
Bounding all terms by the abs completes the proof.
\end{proof}

\begin{proof}[Proof of \cref{lem:sto lower}]
Let $\bbP = \bbP_{\cE_0}$ be the measure induced by the algorithm under environment $\cE_0$ and $\bbQ=\frac{1}{k}\sum_{i=1}^k\bbP_{\cE_i}$ the mixture of measures induced by the remaining $k$ environments.
We have for any event $E$
\begin{align*}
    \min_{i\in[k]}\bbP_{\cE_i}(E)-
    \bbP_{\cE_0}(E) \leq \bbQ(E)-
    \bbP(E) \leq \frac{1}{2}\|\bbQ-\bbP\|_{TV}\,.
\end{align*}
Since the random variables at time $t$ are independent of any previous time step, we can assume that the environment samples $N$ i.i.d.\ full-information samples $(Z,X_1)=((z_s)_{s=1}^N,(x_{1,s})_{s=1}^N)$ and $T$ i.i.d.\ contexts $Y=(x_s)_{s=1}^T$ ahead of time.
At any time $t$, if the agent chooses to play a revealing action, he observes the
next tuple in $(Z,X_1)$, while if the agent does not play a revealing action, he observes the next element in $Y$. 
This construction crucially relies on the agent deciding whether to play a revealing action at time $t$ independently of $x_t$.
Additional information strictly increases the total variation, so we can assume the agent always observes the full collection of random variables at the end
\[\|\bbQ-\bbP\|_{TV}\leq\|\bbQ_{Z,X_1,Y}-\bbP_{Z,X_1,Y}\|_{TV}\,.\]
Applying \cref{lem:tv_ineqs} and observing that $\bbQ_{X_1,Y}=\bbP_{X_1,Y}$, since the contexts are sampled from the same distribution in any environment, we have
\begin{align*}
    \|\bbQ_{Z,X_1,Y}-\bbP_{Z,X_1,Y}\|_{TV} &\leq \E_{X_1,Y\sim \bbQ}\left[\|\bbQ_{Z|X_1,Y}-\bbP_{Z|X_1,Y}\|_{TV}\right]+\|\bbQ_{X_1,Y}-\bbP_{X_1,Y}\|_{TV}\\
    &=\|\bbQ_{Z}-\bbP_{Z}\|_{TV}\,,
\end{align*}
where the last step follows from independence of $Z$ and $X_1,Y$.
Thus we can restrict the problem to bounding the TV over $Z$.

Let $\Omega = \{0,1\}^{k\times N}$ be the outcome space of $Z$, we have
\begin{align*}
    \frac{1}{2}\|\bbQ_Z-\bbP_Z\|_{TV} =  \sum_{Z\in\Omega}\bbI\left\{\bbQ(Z)>\bbP(Z)\right\}\left(\bbQ(Z)-\bbP(Z)\right)
\end{align*}
For a fixed outcome $Z\in\Omega$, denote $n_i = \sum_{t=1}^Nz_{i,t}$, the sum of losses of policy $\pi_i$.
We have
\begin{align*}
    &\bbP(Z)=\bbP_{\cE_0}(Z)=\left(\frac{1}{2}\right)^{Nk} \mbox{ and } \bbP_{\cE_i}(Z)=\left(\frac{1}{2}\right)^{Nk}(1-\Delta)^{n_i}(1+\Delta)^{N-n_i}\,,\mbox{ hence}\\
    &\bbQ(Z)=\sum_{i=1}^k\frac{1}{k}\bbP_{\cE_i}(Z)>\bbP(Z)\,\Leftrightarrow \sum_{i=1}^k \frac{1}{k}\left(\frac{1-\Delta}{1+\Delta}\right)^{n_i} > (1+\Delta)^{-N}\,.
\end{align*}
Denote $\kappa=\ln\left(\frac{1-\Delta}{1+\Delta}\right)$.
Due to symmetry, for any $i\in[k]$: 
\[
\bbP_{\cE_i}\left(\sum_{h=1}^k\exp(n_h\kappa)>k(1+\Delta)^{-N}\right)=
\bbQ\left(\sum_{h=1}^k\exp(n_h\kappa)>k(1+\Delta)^{-N}\right)\,,
\]
hence
\begin{align*}
\frac{1}{2}\|\bbQ_Z-\bbP_Z\|_{TV} &\leq \bbP_{\cE_1}\left(\sum_{h=1}^k\exp(n_h\kappa)>k(1+\Delta)^{-N}\right)-
\bbP_{\cE_0}\left(\sum_{h=1}^k\exp(n_h\kappa)>k(1+\Delta)^{-N}\right)\\
&= \bbP_{\cE_0}\left(\sum_{h=1}^k\exp(n_h\kappa)\leq k(1+\Delta)^{-N}\right)-\bbP_{\cE_1}\left(\sum_{h=1}^k\exp(n_h\kappa)\leq k(1+\Delta)^{-N}\right)\,.
\end{align*}

Next we use the formula for the MGF of a Binomial r.v. $B(N,p)$ given by
\begin{align*}
    \mathbb{E}\left[\exp(t\kappa k)\right] = (p\exp(t\kappa) + (1-p))^N,
\end{align*}
to compute the expectation, variance and third moment of the r.v. $\exp(n_i\kappa)$, where $n_i$ either follows $\cE_0$ or $\cE_1$. Finally we will use the Berry-Essen inequality. We first compute for $m\in\bbN$
\begin{align*}
    &\E_{\cE_j}[\exp(n_im\kappa)] =\begin{cases} \left(\frac{(1+\Delta)^m+(1-\Delta)^m}{2(1+\Delta)^m}\right)^N\mbox{ if }i\neq j\\
    \left(\frac{(1+\Delta)^{m+1}+(1-\Delta)^{m+1}}{2(1+\Delta)^{m}}\right)^N\mbox{ otherwise.}
    \end{cases}
\end{align*}
By using $\Delta^2 N \geq \frac{\ln(k)}{40}\geq \frac{1}{2}$
\begin{align*}
    \mathbb{E}_{\cE_0}\left[\exp(n_i \kappa)\right] &= \frac{1}{(1+\Delta)^N}\\
    \textrm{Var}_{\cE_0}[\exp(n_i \kappa)] &= \frac{(1+\Delta^2)^N - 1}{(1+\Delta)^{2N}}\geq  \frac{\Delta^2N}{(1+\Delta)^{2N}}\geq  \frac{1}{2(1+\Delta)^{2N}}\\
     \mathbb{E}_{\cE_0}\left[\exp(3 n_i \kappa)\right] &= \left(\frac{1+3\Delta^2}{(1+\Delta)^3}\right)^N\leq \frac{\exp(3\Delta^2N)}{(1+\Delta)^{3N}}\,,
\end{align*}
where we have used the facts that $(1+x)^a \leq \exp(ax)$ and $(1+x)^a \geq 1+ax$, for $a\geq 1$.
The Berry-Essen inequality together with Claim~\ref{claim:moment_ineq} now imply that 
\begin{align*}
    \mathbb{P}_{\cE_0}\left[\sum_{i=1}^k \exp(\kappa n_i) \leq k(1+\Delta)^{-N}\right] \leq \frac{1}{2} + \frac{8\exp(3N\Delta^2)}{\sqrt{k}}\,.
\end{align*}

Next we compute the conditional expectation, variance and third moment for $\cE_1$.
\begin{align*}
    \mathbb{E}_{\cE_1}[\exp(n_1\kappa)] &=  \left(\frac{1+\Delta^2}{1+\Delta}\right)^N\\
    \textrm{Var}_{\cE_1}[\exp(n_1\kappa)] &= \left(\frac{1+3\Delta^2}{(1+\Delta)^2}\right)^N-\left(\frac{1+\Delta^2}{1+\Delta}\right)^{2N}\geq 0\\
    \mathbb{E}_{\cE_1}\left[\exp(3n_1\kappa)\right] &=  \left(\frac{1+4\Delta^2+\Delta^4}{(1+\Delta)^3}\right)^N\leq \frac{\exp(5\Delta^2N)}{(1+\Delta)^{3N}}
\end{align*}
Let $Y_j = \exp(\kappa n_j) - \mathbb{E}_{\cE_1}[\exp(\kappa n_j)]$, $\gamma = \frac{(1+\Delta^2)^N-1}{(1+\Delta)^N}$. Then we have 
\begin{align*}
    -\mathbb{P}_{\cE_1}\left[\sum_{j=1}^k \exp(\kappa n_j) \leq k(1+\Delta)^{-N}\right] &= - \mathbb{P}_{\cE_1}\left[\frac{\sum_{j=1}^k Y_j}{\sqrt{\sum_{j=1}^k \textrm{Var}_{\cE_1}(Y_j)}} \leq -\frac{\gamma}{\sqrt{\sum_{j=1}^k \textrm{Var}_{\cE_1}(Y_j)}}\right]\\
    &\leq -\Phi\left(-\frac{\gamma}{\sqrt{\sum_{j=1}^k \textrm{Var}_{\cE_1}(Y_j)}}\right) + \frac{8\exp(5N\Delta^2)}{\sqrt{k-1}},
\end{align*}
where in the inequality we have used the Berry-Essen inequality. 
To bound the remaining term, we have
\begin{align*}
    \frac{\gamma}{\sqrt{\sum_{j=1}^k \textrm{Var}_{\cE_1}(Y_j)}} \leq \sqrt{\frac{(1+\Delta^2)^N-1}{k-1}}\,.
\end{align*}
Let $X$ be a standard Normal r.v., then to bound $-\Phi\left(-\frac{\gamma}{\sqrt{\sum_{j=1}^k \textrm{Var}_{\cE_1}(Y_j)}}\right)$ we have
\begin{align*}
    \Phi\left(-\frac{\gamma}{\sqrt{\sum_{j=1}^k \textrm{Var}_{\cE_1}(Y_j)}}\right) = \mathbb{P}\left(X \geq \frac{\gamma}{\sqrt{\sum_{j=1}^k \textrm{Var}_{\cE_1}(Y_j)}}\right) \geq \frac{1}{2} -  \sqrt{\frac{(1+\Delta^2)^N-1}{k-1}},
\end{align*}
where in the inequality we decomposed the tail probability into the integral from $0$ to $\infty$ minus the integral from $0$ to $\sqrt{\frac{\gamma}{k}}$. Combining the above bounds we have that the TV is bounded by
\begin{align*}
    \frac{1}{2}\|\bbQ_Z-\bbP_Z\|_{TV} \leq 8\exp(3\Delta^2N-\ln(k)/2)+8\exp(5\Delta^2N-\ln(k-1)/2)+\exp(\Delta^2N-\ln(k-1)/2).
\end{align*}
Finally, we use
$\Delta^2N\leq \lfloor\frac{\ln(k-1)}{20\Delta^2}\rfloor \leq \frac{\ln(k-1)}{20}$ and $k>e^{20}$ to obtain
\begin{align*}
    \frac{1}{2}\|\bbQ_Z-\bbP_Z\|_{TV} \leq 17\exp(-\ln(k-1)/4)= \frac{17}{\sqrt[4]{k-1}}\leq \frac{1}{4}\,.
\end{align*}
\end{proof}

With Lemma~\ref{lem:sto lower} we are ready to give the proof of Theorem~\ref{thm:stochastic simple}.
\begin{proof}[Proof of \cref{thm:stochastic simple}]

Let $c_1=\frac{1}{160}, c_2=\frac{1}{10}c_1^{-2}$.
We assume w.l.o.g. that $c_2\C^2\leq\ln(k)\leq \frac{1}{2}T$.
If $\ln(k)=\cO(\C^2)$, then the regret is lower bounded by the regret of $\Pi_1$, if $\ln(k) = \Omega(T)$, the optimal regret is linear in $T$ anyway.
    Pick 
     $\Delta = \min\{\frac{c_1\ln(k)}{\C\sqrt{T}}, \frac{1}{4}\}$. This choice implies $N=\lfloor\frac{\ln(k)}{20\Delta^2}\rfloor\leq \frac{T}{2}$.
    Denote $\cN$ the random number of plays in action $\{1,2\}$, i.e. the number of observations of the full information game. The regret in environment $\cE_0$ is given by
    \begin{align*}
        \Reg_{\cE_0}(T,\Pi_1) = \E_{\cE_0}[\cN]\frac{\Delta}{4} \leq \C\sqrt{ T}\leq \frac{c_1\ln(k)}{\Delta}\,.
    \end{align*}
    Hence
    \begin{align*}
    \E_{\cE_0}[\cN] \leq 4\frac{c_1\ln(k)}{\Delta^2}\leq 80c_1 N\,.
    \end{align*}
    Given any algorithm $\cA$, we create a modified algorithm $\underline\cA$ that uses a stopping time to commit to action $3$ after it played $N$ times action $\{1,2\}$.
    By Markov inequality, the probability of $\cA$ hitting the stopping time on $\cE_0$ is bounded by $80c_1\leq \frac{1}{2}$. Denote that event by $E$. By choice of our constants, the conditions for \cref{lem:sto lower} are met, which implies
    \begin{align*}
        \bbP_{\cE_{i^*}}(E):=\min_{i\in[k]} \bbP_{\cE_i}(E) \leq \frac{1}{2}+\frac{1}{4}=\frac{3}{4}\,.
    \end{align*}
    Hence
    \begin{align*}
 \Reg_{\cE_{i^*}}(T,\Pi_2,\cA)\geq \frac{1}{4}\Reg_{\cE_{i^*}}(T,\Pi_2, \underline{\cA})\geq \frac{1}{4}(T-N)\frac{\Delta}{4}\geq \frac{\Delta T}{32} =\Omega\left(\min\left\{\frac{\ln(k)\sqrt{T}}{\C},T\right\}\right)\,.
    \end{align*}
\end{proof}

\subsection{Stochastic lower bound for improper algorithms and proof of Theorem~\ref{thm: improper}}
\label{app:lower_bounts_proper_alg}

We first show the counterpart to Lemma~\ref{lem:sto lower} for improper algorithms. The proof uses Lemma~\ref{lem:sto lower} together with a sort of a union bound over all possible realizations of time steps at which an improper algorithm chooses to observe the full information game.

\begin{proof}[Proof of Lemma~\ref{lem:sto lower improper}]
As in the proof of \cref{lem:sto lower}, we denote $\bbP=\bbP_{\cE_0}$ the measure induced by running the algorithm in environment $\cE_0$ and $\bbQ=\frac{1}{k}\sum_{i=1}^k\bbP_{\cE_i}$ the mixture over measures induced by environments $i\in[k]$.
$\bbP$ and $\bbQ$ are measures over all possible outcomes of the random variables  $Z=(z_t)_{t=1}^T,X=(x_t)_{t=1}^T,\cT=(\tau_i)_{i=1}^N$.
Let for any subset $M\subset[T]$ denote $Z(M)=(z_t)_{t\in M}$, then the observations of the algorithms are $X,\cT,Z(\cT)$. 
We have
\begin{align*}
    &\TVQP{X,\cT,Z(\cT)}\\
    &=\sum_{M\in\cN,z\in\{0,1\}^{Nk},x\in\{1,2\}^{Tk}} | \bbQ(\cT=M, Z(M)=z, X=x)-\bbP(\cT=M, Z(M)=z, X=x)|\\
    &=\sum_{M\in\cN,z\in\{0,1\}^{Nk},x\in\{1,2\}^{Tk}}\Big( | \bbQ(\cT=M| Z(M)=z, X=x)\cdot\bbQ(Z(M)=z, X=x)\\
    &\qquad\qquad\qquad\qquad\qquad\qquad\qquad-\bbP(\cT=M| Z(M)=z, X=x)\cdot\bbP( Z(M)=z, X=x)|\Big)\\
    &=\sum_{M\in\cN,z\in\{0,1\}^{Nk},x\in\{1,2\}^{Tk}}\bbQ(\cT=M| Z(M)=z, X=x) | \bbQ(Z(M)=z, X=x)-\bbP( Z(M)=z, X=x)|\\
    &\leq\sum_{M\in\cN}\TVQP{Z(M),X}\,,
\end{align*}
where we use in the third equality that both measures are induced by the same algorithm and hence the conditional probabilities satisfy 
\[\bbQ(\cT=M| Z(M)=z, X=x)=\bbP(\cT=M| Z(M)=z, X=x)\,.\]
We use Lemma~\ref{lem:sto lower} to finish the proof.
\end{proof}

\begin{proof}[Proof of \cref{thm: improper}]
    Let $\Delta = \frac{1}{4}$ and $N = \lceil32\C\sqrt{T}\rceil$. Following the proof of Theorem~\ref{thm:stochastic simple} we can reduce the problem to a player which plays an algorithm that commits to action $0$ after $N$ observations of the full game. Using Lemma~\ref{lem:sto lower improper} we can bound
    \begin{align*}
        \frac{1}{2}\TVQP{\cT,Z(\cT),X} \leq 17\exp\left(N\log(T)  - \frac{\log(k-1)}{4}\right). 
    \end{align*}
    It is sufficient to set $k = 1 + 68\exp(4(1+32\C\log(T)\sqrt{T})) $ to ensure that then $\frac{1}{2}\TVQP{\cT,Z(\cT),X} \leq 1/4$. We can then proceed as in the proof of Theorem~\ref{thm:stochastic simple}.
\end{proof}

\subsection{Impossibility of second order bounds}

\begin{proof}[Proof of \cref{thm:yoav contradiction}]
 Given an algorithm $\cA$ satisfying \cref{eq:yoav harder}, we construct the following wrapper.
 Initialize $\cA$ with $\Pi_2$ augmented with a total of $k$ copies of $\pi_0$.
 At any time $t$ receive $p_t\in\Delta(\Pi_2)$ (collapsing the copies of $\pi_0$). With probability $\gamma=3T^{-1/2+\alpha/2}$, play one of the 3 actions uniformly at random (we do this by following $\pi_0$ with probability $\gamma/3$ and following $\pi_1$ with probability $2\gamma/3$, ensuring a proper algorithm), otherwise sample $\pi_t\sim p_t$ and play $A_t=\pi_t(x_t)$. Observe $\ell_{t,A_t}$ and construct loss estimator 
 \begin{align*}
     \hat\ell_t = \frac{\sum_{\pi\in\Pi_2:\pi(x_t)=A_t}\ell_{t,A_t}\mathbf{e}_{\pi}}{\gamma/3+(1-\gamma)\sum_{\pi\in\Pi_2:\pi(x_t)=A_t}p_{t,\pi}} \,.
 \end{align*}
  Finally feed $\hat\ell_t\cdot\gamma/3$ to $\cA$.
 By construction, the wrapper is \emph{proper} and the loss range for the base algorithm is bounded in $[0,1]$.
 The variance terms encountered by $\cA$ are
 \begin{align*}
     \bbV_{\pi\sim p_t}[\frac{\hat\ell_{t,\pi}\gamma}{3}]&\leq \frac{\gamma^2}{9}\sum_{i=0}^k p_{t,\pi_i}\left(\frac{\bbI\{\pi_i(x_t)=A_t\}\ell_{t,A_t}}{\gamma/3+(1-\gamma)\sum_{j:\pi_i(x_t)=\pi_j(x_t)} p_{t,\pi_j}}\right)^2\\
     &\leq \frac{\gamma^2}{9(1-\gamma)(\gamma/3+(1-\gamma)\sum_{j:\pi_j(x_t)=A_t}p_{t,\pi_j})}\,,
 \end{align*}
 and in expectation over the choice of $A_t$, we have
 \begin{align*}
     \E[ \bbV_{\pi\sim p_i}[\frac{\hat\ell_{t,\pi}\gamma}{3}]] \leq \frac{\gamma^2}{9(1-\gamma)}K \leq \frac{2}{3}\gamma^2\,.
 \end{align*}
 The regret in environment $\cE_0$ satisfies due to unbiasedness for selecting the top $1/2$ of all policies (which are the copies of $\pi_0$)
 \begin{align*}
     \Reg(T,\Pi_1) &= \E\left[\sum_{t=1}^T\ell_{t,A_t}-\hat\ell_{t,\pi_0(x_t)}\right]\leq \gamma T + \frac{3(1-\gamma)}{\gamma}\E\left[\sum_{t=1}^T\ip{ p_t-u_{\frac{1}{2}},\gamma\hat\ell_t/3}\right]\\
     &=\gamma T+\frac{3(1-\gamma)}{\gamma}\E\left[\tilde\cO\left(\sqrt{\sum_{t=1}^T\left(\bbV_{\pi\sim p_t}[\gamma\hat\ell_{t,\pi}/3]+T^{\alpha}\right)\ln(2)}\right)\right]
     =\tilde\cO\left(T^{(1+\alpha)/2}\right)\,,
 \end{align*}
 where the last step is by Jensen's inequality.
 Analogously, for any environment $\cE_{i},i\in[k]$, we have for $\epsilon = 1/(2k)$
 \begin{align*}
     \Reg(T,\Pi_2) &
     =\tilde\cO\left(T^{(1+\alpha)/2}+\sqrt{T\ln(|\Pi_2|)}\right)\,.
 \end{align*}
 Following the same argument as in the proof of \cref{cor:main}, an upper bound of $\tilde\cO(T^{(1+\alpha)/2})$ implies $\C = \tilde\Theta(T^{\alpha/2})$, hence by the proof of \cref{thm:stochastic simple}, we have
 \begin{align*}
     \Reg(T,\Pi_2) = \tilde\Omega\left(\frac{\ln(|\Pi_2|)}{\C}\sqrt{T}\right)=\tilde\Omega\left(\ln(|\Pi_2|)T^{(1-\alpha)/2}\right)\,.
 \end{align*}
 This leads to a contradiction for $\ln(|\Pi_2|)=\Omega(T^{(1+\alpha)/2})$.
\end{proof}

\end{document}